%% file: 00-NVGF.tex
\def\ifundefined#1{\expandafter\ifx\csname#1\endcsname\relax}
\let\arXiv = 1

\ifundefined{arXiv}
    \documentclass{article}
    \usepackage{spconf}
\else
    \documentclass[10pt]{article}
\fi

\input{auxFiles/standardPackageInclusion.tex}

\input{auxFiles/auxCommands.tex}
\input{auxFiles/symbolDef.sty}
\input{auxFiles/berkeleyColors.sty}

\usepackage{textcomp} 
\usepackage{IEEEtrantools} 
\usepackage{balance} 
\usepackage{mathtools} 
\usepackage{comment} 
\usepackage{booktabs} 
\ifundefined{arXiv}
\else
    \usepackage{geometry}[margin = 1in]
    \usepackage{setspace}
    \newcommand\blfootnote[1]{%
        \begingroup
        \renewcommand\thefootnote{}\footnote{#1}%
        \addtocounter{footnote}{-1}%
        \endgroup
    }
\fi

\makeatletter
\def\munderbar#1{\underline{\sbox\tw@{$#1$}\dp\tw@\z@\box\tw@}}
\makeatother

\def\mtfnH{\boldsymbol{H}} 
\def\vcfna{\boldsymbol{a}}
\def\vcfnth{\boldsymbol{\tilde{h}}}

\def\thetitle{Node-Variant Graph Filters in Graph Neural Networks}
\ifundefined{arXiv}
    \title{\MakeUppercase{\thetitle}}
\else
    \title{\thetitle}
\fi

\ifundefined{arXiv}
    \name{Fernando Gama$^{\ast}$, Brendon G. Anderson$^{\dag}$, Somayeh Sojoudi$^{\dag}$}
    \address{$\ast$ Department of Electrical and Computer Engineering, Rice University, Houston, TX \\
        $\dag$ Department of Electrical Engineering and Computer Sciences, University of California, Berkeley, CA
    \thanks{This work is partially supported by NSF and ONR.}}
\else
    \author{Fer\hspace{0.015cm}nando Gama, Brendon G. Anderson, Somayeh Sojoudi}
    \date{}
\fi

\begin{document}

\ifundefined{arXiv}
    \ninept
\else
\fi


\maketitle


%
\begin{abstract}
Graph neural networks (GNNs) have been successfully employed in a myriad of applications involving graph signals. Theoretical findings establish that GNNs use nonlinear activation functions to create low-eigenvalue frequency content that can be processed in a stable manner by subsequent graph convolutional filters. However, the exact shape of the frequency content created by nonlinear functions is not known and cannot be learned. In this work, we use node-variant graph filters (NVGFs) --which are linear filters capable of creating frequencies-- as a means of investigating the role that frequency creation plays in GNNs. We show that, by replacing nonlinear activation functions by NVGFs, frequency creation mechanisms can be designed or learned. By doing so, the role of frequency creation is separated from the nonlinear nature of traditional GNNs. Simulations on graph signal processing problems are carried out to pinpoint the role of frequency creation.
\end{abstract}
%

\ifundefined{arXiv}
    \vspace{-0.2cm}
\else
    \blfootnote{This work is partially supported by NSF and ONR. F. Gama is with the Dept. of Electrical and Computer Engineering, Rice University, Houston, TX. Email: \url{fgama@rice.edu}. B. G. Anderson and S. Sojoudi are with the Dept. of Electrical Engineering and Computer Sciences, University of California, Berkeley, CA.}
\fi

\section{Introduction} \label{sec:intro}

\ifundefined{arXiv}
    \vspace{-0.2cm}
\fi

\input{01-introNVGF.tex}


\section{The Node-Variant Graph Filter} \label{sec:NVGF}

\input{02-NVGF.tex}


\section{Frequency Analysis} \label{sec:freq}

\input{03-frequency.tex}


\section{Approximating Activation Functions} \label{sec:activation}

\input{04-activation.tex}


\section{Graph Neural Network Architectures Using Node-Variant Graph Filters}
\label{sec:archit}

\input{05-archit.tex}


\section{Numerical Experiments} \label{sec:sims}

\input{06-simsNVGF.tex}

\section{Conclusion} \label{sec:conclusions}

\input{07-conclusionsNVGF.tex}


\bibliographystyle{bibFiles/IEEEtranD}
\bibliography{bibFiles/myIEEEabrv,bibFiles/biblioNVGF}


\clearpage
\pagenumbering{arabic}
\renewcommand*{\thepage}{A\arabic{page}}
\newpage

\ifundefined{arXiv}
\begin{center}
    {\Large Supplementary Material}
\end{center}
\else
\fi

\appendix

\ifundefined{arXiv}
\input{AC-appendixShort.tex}
\else
\input{AA-appendixProofs.tex}

\input{AB-appendixSims.tex}
\fi

\end{document}

%% file: auxFiles/standardPackageInclusion.tex
\usepackage{xcolor} 

\usepackage{cite} 

\usepackage[pdftex]{graphicx}
\graphicspath{{../figures/}}
\DeclareGraphicsExtensions{.pdf} 

\usepackage{amsmath}
\usepackage{amsthm}
\usepackage{amsfonts}
\usepackage{amssymb}

\usepackage{array} 

\usepackage[caption=false,font=footnotesize]{subfig}

\usepackage{stfloats} 

\usepackage{url} 

%% file: auxFiles/auxCommands.tex
\makeatletter
    \@ifpackageloaded{xcolor}{}{\usepackage{xcolor}}
\makeatother
\makeatletter
    \@ifpackageloaded{needspace}{}{\usepackage{needspace}}
\makeatother

\newcommand{\titledParagraph}[1]{\needspace{1\baselineskip}\noindent {\bf #1}}

\newcommand{\uppercaseGreek}[1]{%
    \begingroup
    \ucmathlist\MakeUppercase{#1}%
    \endgroup
}

\newcommand{\ucmathlist}{%
    \def\alpha{\mathrm{A}}%
    \def\beta{\mathrm{B}}%
    \let\gamma=\Gamma
    \let\delta=\Delta
    \def\epsilon{\mathrm{E}}%
    \def\varepsilon{\mathrm{E}}%
    \def\zeta{\mathrm{Z}}%
    \def\eta{\mathrm{H}}%
    \let\theta=\Theta
    \let\vartheta=\Theta
    \def\iota{\mathrm{I}}%
    \def\kappa{\mathrm{K}}%
    \let\lambda=\Lambda
    \def\mu{\mathrm{M}}%
    \def\nu{\mathrm{N}}%
    \let\xi=\Xi
    \let\pi=\Pi
    \let\varpi=\Pi
    \def\rho{\mathrm{P}}%
    \def\varrho{\mathrm{P}}%
    \let\sigma=\Sigma
    \def\tau{\mathrm{T}}%
    \let\upsilon=\Upsilon
    \let\phi=\Phi
    \let\varphi=\Phi
    \def\chi{\mathrm{X}}%
    \let\psi=\Psi
    \let\omega=\Omega
}

\makeatletter
    \@ifpackageloaded{amsthm}{}{

        \usepackage{amsthm}
    }
\makeatother
\theoremstyle{plain}
    \newtheorem{theorem}{Theorem}
    \newtheorem{proposition}[theorem]{Proposition}
    \newtheorem{lemma}[theorem]{Lemma}
    \newtheorem{corollary}[theorem]{Corollary}
\theoremstyle{definition}

\makeatletter
\def\renewtheorem#1{%
    \expandafter\let\csname#1\endcsname\relax
    \expandafter\let\csname c@#1\endcsname\relax
    \gdef\renewtheorem@envname{#1}
    \renewtheorem@secpar
}
\def\renewtheorem@secpar{\@ifnextchar[{\renewtheorem@numberedlike}{\renewtheorem@nonumberedlike}}
\def\renewtheorem@numberedlike[#1]#2{\newtheorem{\renewtheorem@envname}[#1]{#2}}
\def\renewtheorem@nonumberedlike#1{  
    \def\renewtheorem@caption{#1}
    \edef\renewtheorem@nowithin{\noexpand\newtheorem{\renewtheorem@envname}{\renewtheorem@caption}}
    \renewtheorem@thirdpar
}
\def\renewtheorem@thirdpar{\@ifnextchar[{\renewtheorem@within}{\renewtheorem@nowithin}}
\def\renewtheorem@within[#1]{\renewtheorem@nowithin[#1]}
\makeatother

%% file: 01-introNVGF.tex


\noindent Graph neural networks (GNNs) \cite{Bronstein2017-GeometricDL, Gama2020-GNN} are learning architectures that have been successfully applied to a wide array of graph signal processing (GSP) problems ranging from recommendation systems \cite{Monti2017-RecommendationGNN, Levie2018-CayleyNets} and authorship attribution \cite{Segarra2015-Authorship, Gama2019-Archit}, to physical network problems including wireless communications \cite{Eisen2020-WirelessEGNN}, control \cite{Gama2021-LQRGNN}, and sensor networks \cite{Owerko2018-Sensor}. In the context of GSP, frequency analysis has been successful at providing theoretical insight into the observed success of GNNs \cite{Gama2020-Stability, Kenlay2021-StabilityRewiring, Levie2020-TransferabilitySpectral} and graph scattering transforms \cite{Wolf2019-GeometricScattering, Lerman2020-Scattering}.

Of particular interest is the seminal work by Mallat \cite{Mallat2012-Scattering, Bruna2013-Scattering}, concerning discrete-time signals and images, that argues that the improved performance of convolutional neural networks (CNNs) over linear convolutional filters is due to the activation functions. Concretely, nonlinear activation functions allow for high-frequency content to be spread into lower frequencies, where it can be processed in a stable manner---a feat that cannot be achieved by convolutional filters alone\ifundefined{arXiv}.\else, which are unstable when processing high frequencies.\fi
Leveraging the notion of graph Fourier transform (GFT) \cite{Sandryhaila2014-DSPGfreq}, these results have been extended to GNNs \cite{Gama2020-Stability}, establishing that the use of functions capable of creating low-eigenvalue frequency content allows them to be robust to changes in the graph topology, facilitating scalability and transferability \cite{Ruiz2021-GNNs}.

While nonlinear activation functions play a key role in the creation of low-eigenvalue frequency content, it is not possible to know the exact shape in which this frequency content is actually generated. Node-variant graph filters (NVGFs) \cite{Segarra2017-GraphFilterDesign}, which essentially assign a different filter tap to each node in the graph, are also able to generate frequency content. Different from nonlinear activation functions, this frequency content can be exactly computed, given the filter taps. Thus, by learning or carefully designing these filter taps, it is possible to know exactly how the frequency content is being generated.

The NVGF is a linear filter, which means that replacing the nonlinear activation functions with NVGFs actually renders the whole architecture a linear one. Therefore, by comparing the performance of this architecture to that of a traditional GNN, it is possible to isolate the role of frequency creation in the overall performance of the architecture, from that of the nonlinear nature of mappings. The contributions of this paper can be summarized as follows:

\ifundefined{arXiv}
\begin{list}{}{
        \setlength{\labelwidth}{20pt}
        \setlength{\leftmargin}{2pt}
        \setlength{\labelsep}{1pt}
        \setlength{\itemsep}{-1pt}
        \setlength{\topsep}{-1pt}
        \setlength{\parskip}{-1pt}
    }

    \item[1.] We introduce NVGFs as a means of replacing nonlinear activation functions, motivated by their ability to create frequencies. We obtain closed-form expressions for the frequency response of NVGFs as a function of the filter taps.
    \item[2.] We prove that NVGFs are Lipschitz continuous with respect to changes in the underlying graph topology.
    \item[3.] We put forth a framework for designing NVGFs.
    \item[4.] We propose a GNN architecture where the nonlinear activation function is replaced by a NVGF. The filter taps of the NVGF can be either learned or designed. The resulting architecture decouples the role of frequency creation from the nonlinear nature of the GNN.
    \item[5.] We investigate the problem of authorship attribution to demonstrate, both quantitatively and qualitatively, the role of frequency creation in the performance of a GNN, and its relationship to the nonlinear nature of traditional architectures.

\end{list}
\else
\begin{list}{}{
        \setlength{\labelwidth}{20pt}
        \setlength{\leftmargin}{17.5pt}
        \setlength{\labelsep}{1.5pt}
        \setlength{\itemsep}{-1pt}
        \setlength{\topsep}{0pt}
        \setlength{\parskip}{0pt}
    }

    \item[\textbf{(C1)}] We introduce NVGFs as a means of replacing nonlinear activation functions, motivated by their ability to create frequencies. We obtain closed-form expressions for the frequency response of NVGFs as a function of the filter taps.
    \item[\textbf{(C2)}] We prove that NVGFs are Lipschitz continuous with respect to changes in the underlying graph topology.
    \item[\textbf{(C3)}] We put forth a framework for designing NVGFs.
    \item[\textbf{(C4)}] We propose a GNN architecture where the nonlinear activation function is replaced by a NVGF. The filter taps of the NVGF can be either learned from data (Learn NVGF) or obtained by design (Design NVGF). The aim of the resulting architecture is to decouple the role of frequency creation from the nonlinear nature of the GNN.
    \item[\textbf{(C5)}] We investigate the problem of authorship attribution to demonstrate, both quantitatively and qualitatively, the role of frequency creation in the performance of a GNN, and its relationship to the nonlinear nature of traditional architectures.

\end{list}
\fi

In essence, we show that nonlinear activation functions are not strictly required for creating frequencies, as originally thought in \cite{Mallat2012-Scattering, Gama2020-Stability}, but that linear NVGF activation functions are sufficient. Furthermore, we demonstrate that this frequency content can be learned with respect to the specific problem under study. \ifundefined{arXiv}All proofs, as well as the code and further simulations, can be found online\footnote{Proofs: \url{http://arxiv.org/abs/2106.00089} \\ \indent $\ \, \, $ Code: \url{http://github.com/fgfgama/nvgf}}\else All proofs, as well as further simulations, can be found in the appendices. Code can be found online at \url{http://github.com/fgfgama/nvgf}.\fi



\titledParagraph{Related work.} GNNs constitute a very active area of research \cite{Rey2019-EncoderDecoder, Rey2021-Overparametrized}. From a GSP perspective, spectral filtering is used in \cite{Bruna2014-SpectralGNN}, it is then replaced by computationally simpler Chebyshev polynomials \cite{Defferrard2016-ChebNets}, and subsequently by general graph convolutional filters \cite{Gama2019-Archit}. GNNs were also adopted in non-GSP problems \cite{Kipf2017-GCN, Weinberger2019-SGC, Velickovic2018-GAT}. The proposed replacement of nonlinear activation functions by NVGFs creates a linear architecture that uses both convolutional and non-convolutional graph filters.

NVGFs are first introduced in \cite{Segarra2017-GraphFilterDesign} to extend time-variant filters into the realm of graph signals. In that work, NVGFs are used to optimally approximate linear operators in a distributed manner. In this paper, we focus on the frequency response of NVGFs and on their capacity to create frequency content.

Leveraging the notion of GFT, \cite{Gama2020-Stability} shows that a GNN is Lipschitz continuous to small changes in the underlying graph structure. Likewise, frequency analysis has been used to understand graph scattering transforms, where the filters used in the GNN are chosen to be wavelets (and not learned) \cite{Wolf2019-GeometricScattering, Lerman2020-Scattering}.
In this paper, we focus on the role of frequency creation that is put forth in \cite{Mallat2012-Scattering, Gama2020-Stability}, and study NVGFs as linear mechanisms for achieving this.

%% file: 02-NVGF.tex


Let $\stG = (\stV, \stE)$ be an undirected, possibly weighted, graph with a set of $N$ nodes $\stV = \{\lmv_{1},\ldots,\lmv_{N}\}$ and a set of edges $\stE \subseteq \stV \times \stV$. Define a graph signal as the function $\fnx: \stV \to \fdR$ that associates a scalar value to each node. For a fixed ordering of the nodes, the graph signal can be conveniently described by means of a vector $\vcx \in \fdR^{N}$ such that the $i^{\text{th}}$ entry $[\vcx]_{i}$ is the value $x_{i}$ of the signal on node $\lmv_{i}$, i.e., $[\vcx]_{i} = x_{i} = \fnx(\lmv_{i})$. 

Describing a graph signal as the vector $\vcx \in \fdR^{N}$ is mathematically convenient but carries no information about the underlying graph topology that supports it. This information can be recovered by defining a graph matrix description (GMD) $\mtS \in \fdR^{N \times N}$ which is a matrix that respects the sparsity pattern of the graph, i.e., $[\mtS]_{ij} = 0$ for all distinct indices $i$ and $j$ such that $(\lmv_{j},\lmv_{i}) \notin \stE$. Examples of GMDs widely used include the adjacency matrix, the Laplacian matrix, and their corresponding normalizations \cite{Sandryhaila2013-DSPG, Shuman2013-SPG, Ortega2018-GSP}.

The GMD $\mtS$ can thus be leveraged to process the graph signal $\vcx$ in such a way that the underlying graph structure is exploited. The most elementary example is the linear map $\fnS: \fdR^{N} \to \fdR^{N}$ between graph signals given by $\vcy = \fnS(\vcx) = \mtS \vcx$. This linear map is a linear combination of the information located in the one-hop neighborhood of each node:
\begin{equation} \label{eq:graphShiftSingle}
    y_{i} = [\vcy]_{i} = \sum_{j=1}^{N} [\mtS]_{ij} [\vcx]_{j} = \sum_{j : \lmv_{j} \in \stN_{i} \cup \{\lmv_{i}\}} s_{ij}x_{j}
\end{equation}
where $\stN_{i} = \{\lmv_{j}: (\lmv_{j},\lmv_{i}) \in \stE \}$ is the set of nodes that share an edge with node $\lmv_{i}$. The last equality in the above equation is due to the sparsity pattern of the GMD $\mtS$.

More generally, a graph filter $\fnH(\cdot;\mtS): \fdR^{N} \to \fdR^{N}$ is a mapping between graph signals that leverages the structure encoded in $\mtS$ \cite{Segarra2017-GraphFilterDesign}. In particular, linear shift-invariant graph filters (LSIGFs) are those that can be built as a polynomial in $\mtS$:
\begin{equation} \label{eq:LSIfilter}
    \fnH^{\text{lsi}}(\vcx;\mtS) = \sum_{k=0}^{K} h_{k} \mtS^{k} \vcx = \mtfnH^{\text{lsi}}(\mtS) \vcx
\end{equation}
where $\mtfnH^{\text{lsi}}(\mtS) = \sum_{k=0}^{K} h_{k} \mtS^{k}$ with $h_{k} \in \fdR$. Note that $\fnH^{\text{lsi}}(\vcx;\mtS)$ is written in the form of $\mtfnH^{\text{lsi}}(\mtS)\vcx$ to emphasize that the function is linear in the input $\vcx$, i.e., $\vcx$ is multiplied by a matrix $\mtfnH^{\text{lsi}}(\mtS)$ that is parametrized by $\mtS$. LSIGFs inherit their name from the fact that they satisfy the property that $\mtfnH^{\text{lsi}}(\mtS)\mtS \vcx = \mtS \mtfnH^{\text{lsi}}(\mtS) \vcx$. The set of polynomial coefficients $\{h_{k}\}_{k=0}^{K}$ are called filter taps, and can be collected in a vector $\vch \in \fdR^{K+1}$ defined as $[\vch]_{k+1} = h_{k}$ for all $k \in \{0,\ldots,K\}$. Note that the term $\mtS^{k}\vcx$ is a convenient mathematical formulation, but in practice $\mtS^{k}\vcx$ is computed by exchanging information $k$ times with one-hop neighbors, i.e., there are no matrix powers involved. In general, GSP regards $\mtS$ as given by the structure of the problem, and regards $\vcx$ as the actionable data \cite{Sandryhaila2013-DSPG, Shuman2013-SPG, Ortega2018-GSP}.

This paper focuses on NVGFs \cite{Segarra2017-GraphFilterDesign}, which are linear filters that assign a different filter tap to each node, for each application of $\mtS$. This can be compactly written as follows:
\begin{equation} \label{eq:NVGF}
    \fnH^{\text{nv}}(\vcx;\mtS) = \sum_{k=0}^{K} \diag(\vch^{(k)}) \mtS^{k} \vcx = \mtfnH^{\text{nv}}(\mtS) \vcx
\end{equation}
where $\vch^{(k)} \in \fdR^{N}$ and $\diag(\cdot)$ is the diagonal operator that takes a vector and creates a diagonal matrix with the elements of the vector in the diagonal. Since the NVGF is linear in the input, it holds that $\fnH^{\text{nv}}(\vcx;\mtS) = \mtfnH^{\text{nv}}(\mtS) \vcx$, where $\mtfnH^{\text{nv}}(\mtS) = \sum_{k=0}^{K} \diag(\vch^{(k)}) \mtS^{k}$. The $i^{\text{th}}$ entry of the vector, $[\vch^{(k)}]_{i} = h_{ik}$, is the filter tap that node $\lmv_{i}$ uses to weigh the information incoming after $k$ exchanges with its neighbors. The set of all filter taps can be conveniently collected in a matrix $\mtH \in \fdR^{N \times (K+1)}$ where the $(k+1)^{\text{th}}$ column is $\vch^{(k)} \in \fdR^{N}$ and the $i^{\text{th}}$ row, denoted by $\vch_{i} \in \fdR^{K+1}$, contains the $K+1$ filter taps used by node $\lmv_{i}$, i.e., $[\vch_{i}]_{k+1} = h_{ik}$. 

The LSIGF in \eqref{eq:LSIfilter} and the NVGF in \eqref{eq:NVGF} are both linear and local processing operators. They depend linearly on the input graph signal $\vcx$ as indicated by the matrix multiplication notation in \eqref{eq:LSIfilter} and in \eqref{eq:NVGF}. They are local in the sense that, to compute the output, each node requires information relayed directly by their immediate neighbors. The LSIGF is characterized by the collection of $K+1$ filter taps. The NVGF, on the other hand, is characterized by $N(K+1)$ filter taps. It is noted that while the NVGF requires additional memory to store more parameters, this can be distributed throughout the graph. It is also observed that both the LSIGF and the NVGF have the same computational complexity.


%% file: 03-frequency.tex


The GMD $\mtS \in \fdR^{N \times N}$ can be used to define a spectral representation of the graph signal $\vcx \in \fdR^{N}$ \cite{Sandryhaila2014-DSPGfreq}. Since the graph is undirected, assume that $\mtS$ is symmetric so that it can be diagonalized by an orthonormal basis of eigenvectors $\{\vcv_{i}\}_{i=1}^N$, where $\vcv_{i} \in \fdR^{N}$ and $\mtS \vcv_{i} = \lambda_{i} \vcv_{i}$, with $\lambda_{i} \in \fdR$ being the corresponding eigenvalue. Then, it holds that $\mtS = \mtV \diag(\vclambda) \mtV^{\Tr}$, where the $i^{\text{th}}$ column of $\mtV$ is $\vcv_{i}$ and where $\vclambda \in \fdR^{N}$ is given by $[\vclambda]_{i} = \lambda_{i}$ for $i=1,\ldots,N$. We assume throughout this paper that the eigenvalues are distinct, which is typically the case for random connected graphs.

The spectral representation of a graph signal $\vcx$ with respect to its underlying graph support described by $\mtS$ is given by
\begin{equation} \label{eq:GFT}
    \vctx = \vcV^{\Tr} \vcx
\end{equation}
where $\vctx \in \fdR^{N}$, see \cite{Sandryhaila2014-DSPGfreq}. The spectral representation $\vctx$ of the graph signal $\vcx$ contains the coordinates of representing $\vcx$ on the eigenbasis $\{\vcv_{i}\}_{i=1}^N$ of the support matrix $\mtS$. The resulting vector $\vctx$ is known as the frequency response of the signal $\vcx$. The $i^{\text{th}}$ entry $[\vctx]_{i} = \vcv_{i}^{\Tr} \vcx = \sctx_{i} \in \fdR$ of the frequency response $\vctx$ measures how much the $i^{\text{th}}$ eigenvector $\vcv_{i}$ contributes to the signal $\vcx$. The operation in \eqref{eq:GFT} is called the GFT, and thus the frequency response $\vctx$ is often referred to as the GFT of the signal $\vcx$.

The GFT offers an alternative representation of the graph signal $\vcx$ that takes into account the graph structure in $\mtS$. The effect of a filter can be characterized in the spectral domain by computing the GFT of the output. For instance, when considering a LSIGF, the spectral representation of the output $\vcy = \mtfnH^{\text{lsi}}(\mtS) \vcx$ is
\begin{equation} \label{eq:GFToutputLSI}
    \vcty = \mtV^{\Tr} \sum_{k=0}^{K} h_{k} \mtS^{k} \vcx = \sum_{k=0}^{K} h_{k} \diag(\vclambda^{k}) \vctx = \diag(\vcth) \vctx
\end{equation}
where the eigendecomposition of $\mtS$, the GFT of $\vcx$, and the fact that $\mtS^{k} = \mtV \diag(\vclambda^{k}) \mtV^{\Tr}$ were all used. Note that $\vclambda^{k}$ is a shorthand notation that means $[\vclambda^{k}]_{i} = \lambda_{i}^{k}$. The vector $\vcth \in \fdR^{N}$ is known as the frequency response of the filter and its $i^{\text{th}}$ entry is given by
\begin{equation} \label{eq:freqResponse}
    [\vcth]_{i} = \fnth(\lambda_{i}) = \sum_{k=0}^{K} h_{k} \lambda_{i}^{k}
\end{equation}
where $\fnth:\fdR \to \fdR$ is a polynomial defined by the set of filter taps $\{h_{k}\}_{k=0}^K$. The function $\fnth(\cdot)$ depends only on the filter coefficients and not on the specific graph on which it is applied, and thus is valid for all graphs. The effect of filtering on a specific graph comes from instantiating $\fnth(\cdot)$ on the specific eigenvalues of that graph. The function $\fnth(\cdot)$ is denoted as the frequency response as well, and it will be clear from context whether we refer to the function $\fnth(\cdot)$ or to the vector $\vcth$ given in \eqref{eq:freqResponse}\ifundefined{arXiv}\else\
that stems from evaluating $\fnth(\cdot)$ at each of the $N$ eigenvalues $\lambda_{i}$\fi
. Note that since $\fnth(\cdot)$ is an analytic function, it can be applied to the square matrix $\mtS$ so that $\fnth(\mtS) = \mtfnH^{\text{lsi}}(\mtS)$.

In the case of the LSIGF, it is observed from \eqref{eq:GFToutputLSI} that the $i^{\text{th}}$ entry of the frequency response of the output $[\vcty]_{i} = \scty_{i}$ is given by the elementwise multiplication
\begin{equation} \label{eq:convolutionTheorem}
    \scty_{i} = \fnth(\lambda_{i}) \sctx_{i}.
\end{equation}
This implies that the frequency response of the output $\scty_{i}$ is the elementwise multiplication of the frequency response of the filter $\fnth(\lambda_{i})$ and the frequency response of the input $\sctx_{i}$.
This makes \eqref{eq:convolutionTheorem} the analogue of the convolution theorem for graph signals. Therefore, oftentimes the LSIGF in \eqref{eq:LSIfilter} is called graph convolution. It is observed from \eqref{eq:convolutionTheorem} that LSIGFs are capable of learning any type of frequency response (low-pass, high-pass, etc.), but that they are not able to create frequencies, i.e., if $\sctx_{i} = 0$, then $\scty_{i} = 0$.

Unlike discrete-time signals, the frequency response of the LSIGF is not computed in the same manner as the frequency response of graph signals. More specifically, the frequency response of the filter $\vcth$ can be directly computed from the filter taps $\vch$ by means of a Vandermonde matrix $\mtLambda \in \fdR^{N \times (K+1)}$ given by $[\mtLambda]_{ik} = \lambda_{i}^{k-1}$ as follows:
\begin{equation} \label{eq:GFTfilter}
    \vcth = \mtLambda \vch.
\end{equation}
This implies that\ifundefined{arXiv} \else\ when processing graph signals, the graph filter cannot be uniquely represented by a graph signal \cite[Th. 5]{Sandryhaila2013-DSPG}, and thus the concept of impulse response is no longer valid. Furthermore, \fi
the graph convolution is not \ifundefined{arXiv}commutative.\else symmetric, i.e., the filter and the graph signal do not commute. Finally, it is interesting to note that, when $\mtS$ is the adjacency matrix of a directed cycle and $K=N-1$, then $\mtLambda = \mtV^{\Hr}$ and the GFT of the signal and of the filter taps (in time, known as the impulse response) becomes equivalent, as expected.\fi

When considering NVGFs, as in \eqref{eq:NVGF}, the convolution theorem \eqref{eq:convolutionTheorem} no longer holds. Instead, the frequency response of the output is given in the following proposition.

\begin{proposition}[Frequency response of NVGF] \label{prop:GFToutputNVGF}
Let $\vcy = \mtfnH^{\text{nv}}(\mtS) \vcx = \sum_{k=0}^{K} \diag(\vch^{(k)}) \mtS^{k} \vcx$ be the output of an arbitrary NVGF characterized by some filter tap matrix $\mtH \in \fdR^{N \times(K+1)}$. Then, the frequency response $\vcty = \mtV^{\Tr}\vcy$ of the output is given by
\begin{equation} \label{eq:GFToutputNVGF}
    \vcty = \mtV^{\Tr} \big( \mtV \circ (\mtH \mtLambda^{\Tr}) \big) \vctx
\end{equation}
where $\circ$ denotes elementwise product of matrices.
\end{proposition}

It is immediately observed that NVGFs are capable of generating new frequency content, even though they are linear.

\begin{corollary} \label{cor:NVGFnewFreq}
If the matrix $\mtV^{\Tr} \big( \mtV \circ (\mtH \mtLambda^{\Tr}) \big)$ is not diagonal, then the output exhibits frequency content that is not present in the input.
\end{corollary}

\ifundefined{arXiv}
\else
As an example, consider the case where the input is given by a single frequency signal, i.e., $\vcx = \vcv_{t}$ for some $t\in\{1,\ldots,N\}$ so that $\vctx = \vce_{t}$ where $[\vce_{t}]_{i} = 1$ if $i=t$ and $[\vce_{t}]_{i} = 0$ otherwise. This is a signal that has a single frequency component. Yet, the $i^{\text{th}}$ entry of the frequency response of the NVGF output is $\scty_{i} = \sum_{j=1}^{N} \fnth_{j}(\lambda_{t}) v_{ji}v_{jt}$, where $v_{ij} = [\mtV]_{ij}$ and $\fnth_{j}$ is the frequency response obtained from using the $K+1$ coefficients at node $\lmv_{j}$. Unless $\fnth_{j}(\lambda_{t})$ is a constant for all $j$, i.e., $\fnth_{j}(\lambda_{t}) = \scth_{t}$, the $i^{\text{th}}$ entry of the GFT of the output is nonzero even when the $i^{\text{th}}$ entry of the GFT of the input is zero. This also shows that by carefully designing the filter taps for each node, the frequencies that are allowed to be created can be chosen, analyzed, and understood.
\fi

To finalize the frequency analysis of NVGFs, we establish their Lipschitz continuity to changes in the underlying graph support, as decribed by the matrix $\mtS$. In what follows, we denote the spectral norm of a matrix $\mtA$ by $\|\mtA\|_2$.

\begin{theorem}[Lipschitz continuity of the NVGF with respect to $\mtS$] \label{thm:stability}
Let $\stG$ and $\sthG$ be two graphs with $N$ nodes, described by GMDs $\mtS \in \fdR^{N \times N}$ and $\mthS \in \fdR^{N \times N}$, respectively. Let $\mtH \in \fdR^{N \times(K+1)}$ be the coefficients of any NVGF. Given a constant $\sceps > 0$, if $\|\mthS - \mtS\|_{2} \leq \sceps$, it holds that
\begin{equation} \label{eq:stability}
    \Big\| \big( \mtfnH^{\text{nv}}(\mthS) - \mtfnH^{\text{nv}}(\mtS) \big) \vcx \Big\|_{2} \leq \sceps C \sqrt{N} ( 1 + 8N) \| \vcx \|_{2} + \bigOh(\sceps^{2})
\end{equation}
where $\mtfnH^{\text{nv}}(\mtS)$ and $\mtfnH^{\text{nv}}(\mthS)$ are the NVGF on $\mtS$ and $\mthS$, respectively, and where $C$ is the Lipschitz constant of the frequency responses at each node, i.e., $|\fnth_{t}(\lambda_{j})-\fnth_{t}(\lambda_{i})| \leq C |\lambda_{j} - \lambda_{i}|$ for all $i,j,t \in \{1,\ldots,N\}$.
\end{theorem}

Theorem~\ref{thm:stability} establishes the Lipschitz continuity of the NVGF filter with respect to the support matrix $\mtS$ (Lipschitz continuity with respect to the input $\vcx$ is immediately given for bounded filter taps) as long as the graphs $\mtS$ and $\mthS$ are similar, i.e., $\sceps \ll 1$. The bound is proportional to this difference, $\sceps$, and to the shape of the frequency responses at each node through the Lipschitz constant $C$. It also depends on the number of nodes $N$, but it is fixed for given graphs with the same number of nodes. In short, Theorem~\ref{thm:stability} gives mild guarantees on the expected performance of the NVGF across a wide range of graphs $\mthS$ that are close to the graph $\mtS$. 

%% file: 04-activation.tex


One of the main roles of activation functions in neural networks is to create low-frequency content that can be processed in a stable manner \cite{Mallat2012-Scattering}. However, the way the nonlinearities create this frequency content is unknown and cannot be shaped. One alternative for tailoring the frequency creation process to the specific problem under study is to learn the NVGF filter taps (Sec.~\ref{sec:archit}). However, doing so, implies that the number of learnable parameters depend on $N$ which may lead to overfitting. In what follows we propose one method of designing, instead of learning, the NVGFs.

\textbf{Problem statement.} Assume that each data point $\vcx$ is a random graph signal with mean $\xp[\vcx] = \vcmu_{x}$, correlation matrix $\mtR_{x} = \xp[\vcx \vcx^{\Tr}]$, and covariance matrix $\mtC_{x} = \xp[(\vcx-\vcmu_{x})(\vcx-\vcmu_{x})^{\Tr}]$. The objective is to estimate a pointwise nonlinear function $\fnrho: \fdR \to \fdR$ such that $[\fnrho(\vcx)]_{i} = \fnrho([\vcx]_{i})$, using a NVGF-based estimator as $\vchy = \mtfnH^{\text{nv}}(\mtS) \vcx + \vcc$ for $\mtfnH^{\text{nv}}(\mtS)$ as in \eqref{eq:NVGF} and $\vcc \in \fdR^{N}$. Given the random variable $\vcy = \fnrho(\vcx)$, the aim is to find the filter taps $\mtH \in \fdR^{N \times (K+1)}$ that minimize the mean squared error (MSE)
\begin{equation} \label{eq:Hopt}
    \mtH^{\opt} = \argmin_{\mtH \in \fdR^{N \times (K+1)}} \xp \big[ \| \vchy - \vcy \|_{2}^{2} \big].
\end{equation}
Our particular focus is set on obtaining unbiased estimators.

\begin{lemma}[Unbiased estimator] \label{l:unbiased}
Let $\vcmu_{\rho} = \xp[\fnrho(\vcx)]$. The NVGF-based estimator is unbiased if and only if $\vcc = \vcmu_{\rho} - \mtfnH^{\text{nv}}(\mtS) \vcmu_{x}$.
\end{lemma}

From Lemma~\ref{l:unbiased}, the unbiased estimator is now
\begin{equation}
    \vchy = \mtfnH^{\text{nv}}(\mtS) \big( \vcx - \vcmu_{x} \big) + \vcmu_{\rho}.
\end{equation}
Therefore, the objective becomes finding the filter tap matrix $\mtH \in \fdR^{N \times (K+1)}$ for some fixed value of $K$ that satisfies
\begin{equation} \label{eq:approx}
    \mtH^{\opt} = \argmin_{\mtH \in \fdR^{N \times(K+1)}} \xp \Big[ \big\| \mtfnH^{\text{nv}}(\mtS) (\vcx-\vcmu_{x}) - \big(\fnrho(\vcx) - \vcmu_{\rho}\big) \big\|_{2}^{2} \Big].
\end{equation}
Note that, due to the orthogonal nature of the GFT, minimizing \eqref{eq:approx} is equivalent to minimizing the difference of the corresponding frequency responses.

The optimal filter taps for each node, i.e., the rows $\vch_{i}^{\opt} \in \fdR^{K+1}$ of $\mtH^{\opt}$ that solve \eqref{eq:approx}, can be obtained by solving a linear system of equations as follows.

\begin{proposition}[Optimal NVGF] \label{prop:NVGFoptimal}
Let $\vcu_{i}$ denote the $i^{\text{th}}$ row of $\mtV$,  $\mtR_{i} = \mtLambda^{\Tr} \diag(\vcu_{i}) \vcV^{\Tr} \mtC_{x} \vcV \diag(\vcu_{i}) \mtLambda$ be the covariance matrix of the frequency response at node $\lmv_{i}$ for the input $\vcx$, and $\vcp_{i} = \mtLambda^{\Tr} \diag(\vcu_{i}) \vcV^{\Tr} \xp[(\fnrho(x_{i}) - \mu_{\rho i})(\vcx-\vcmu_{x})]$ denote the correlation between the filtered signal and the target nonlinearity. Then, a set of filter taps $\{\vch_{1}^{\opt},\ldots,\vch_{N}^{\opt}\}$ is optimal for \eqref{eq:approx} if and only if they solve the system of linear equations
\begin{equation} \label{eq:NVGFoptimal}
    \mtR_{i} \vch_{i}^{\opt} = \vcp_{i}, \quad i\in\{1,\ldots,N\}.
\end{equation}
\end{proposition}

From Proposition~\ref{prop:NVGFoptimal}, it is immediate that only knowledge of the first and second moments of $\vcx$, and of the correlation between the input $\vcx$ and the output $\vcy = \fnrho(\vcx)$, is required to solve for the optimal NVGF. These moments can be estimated from training data. Also observed is that the optimal filter taps for each node can be computed separately at each node.

\ifundefined{arXiv}
\else
We now consider the specific case of random graph signals with zero-mean, independent, identically distributed (i.i.d.) entries, to illustrate the form of the optimal NVGF.

\begin{corollary}[Zero-mean, i.i.d., ReLU nonlinearity] \label{cor:ReLU}
Assume that the elements of $\vcx$ are i.i.d., and that $\vcmu_{x} = \vcZeros$ and $\mtR_{x} = \mtC_{x} = \sigma_{x}^{2} \mtI$ for some $\sigma_{x}^{2} > 0$. Consider an unbiased NVGF-based estimator of the form
\begin{equation} \label{eq:zero-meanEstimator}
    \schy_{i} = (\scxi^{2}/\sigma_{x}^{2}) x_{i} + \mu_{\rho}
\end{equation}
with $\scxi^{2} = \xp[\rho(x)x]$ and $\mu_{\rho} = \xp[\rho(x)]$, where $x$ is distributed as the elements of $\vcx$. Then, the estimator \eqref{eq:zero-meanEstimator} is optimal for \eqref{eq:approx}. If, additionally, $\rho(\cdot)=\ReLU(\cdot)$ and the distribution of $x$ is symmetric around $0$, then $\mu_{\rho}=0$ and $\scxi^{2} = \sigma_{x}^{2}/2$, so that \eqref{eq:zero-meanEstimator} reduces to $\schy_{i} = x_{i}/2$.
\end{corollary}

Corollary~\ref{cor:ReLU} shows that if the elements of the graph signal are zero-mean and i.i.d., then the NVGF boils down to a LSI graph filter. Hence, this optimal unbiased NVGF is not capable of generating frequencies. This is sensible since the elements of the graph signal $\vcx$ bare no relationship to the underlying graph support, and thus the NVGF does not leverage that support to create the appropriate frequencies. Furthermore, if the distribution of each element of $\vcx$ is symmetric around zero, then an optimal approximation of the $\ReLU$ nonlinearity amounts to a LSI graph filter that outputs half of the input, as one would expect.

\begin{corollary}[Stationary graph processes] \label{cor:stationary}
Let $\vcx$ be a stationary graph process \cite{Perraudin2017-StationaryGSP, Marques2017-StationaryGSP}, so that $\vcmu_{x} = \vcZeros$ and $\mtR_{x} = \mtC_{x} = \mtV \diag(\vcq) \mtV^{\Tr}$ with power spectral density $\vcq \in \fdR^{N}$. Then, NVGF filter taps $\vch_i^{\opt}$ at node $\lmv_{i}$ solving the linear equations $\diag(\vcu_{i} \circ \vcq) \mtLambda \vch_{i}^{\opt} = \mtV^{\Tr} \xp [\rho(x_{i}) \vcx]$ are optimal.
\end{corollary}

It is observed from Corollary~\ref{cor:stationary} that the effect of the contribution of each frequency component to node $\lmv_{i}$ gets modulated by the power spectral density $\vcq$. Also, the filter in this case is different for each node, and thus optimal unbiased NVGFs for processing stationary graph processes actually create frequencies.
\fi

%% file: 05-archit.tex


\begin{figure*}
    \centering
    \subfloat[Performance Comparison]{%
        \label{fig:austen:comparison}%
        \includegraphics[width=0.55\columnwidth]{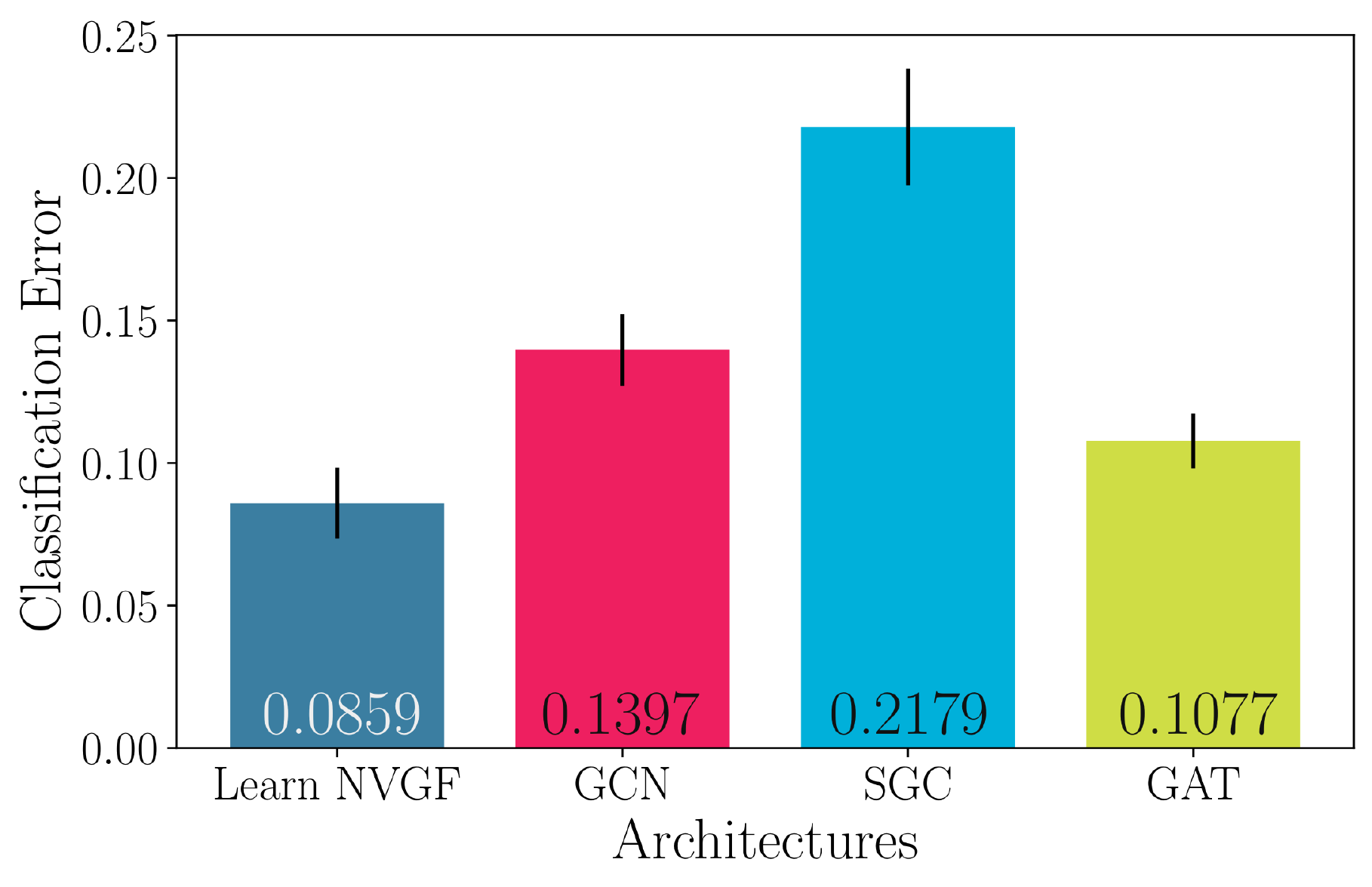}
    }
    \hfill
    \subfloat[Role of frequency creation]{%
        \label{fig:austen:change}
        \includegraphics[width=0.55\columnwidth]{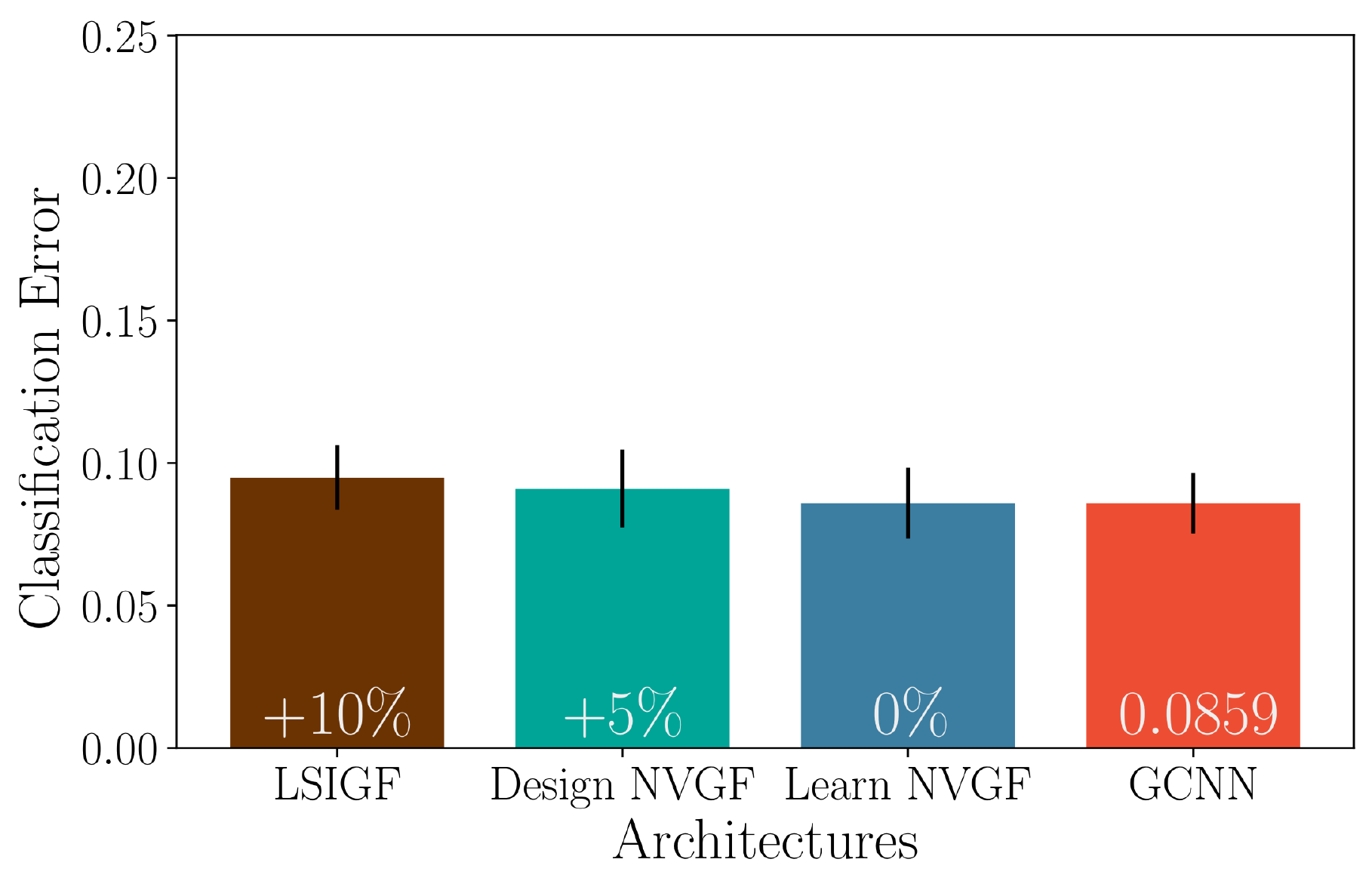}
    }
    \hfill
    \subfloat[Input frequencies]{%
        \label{fig:austen:input}
        \includegraphics[width=0.55\columnwidth]{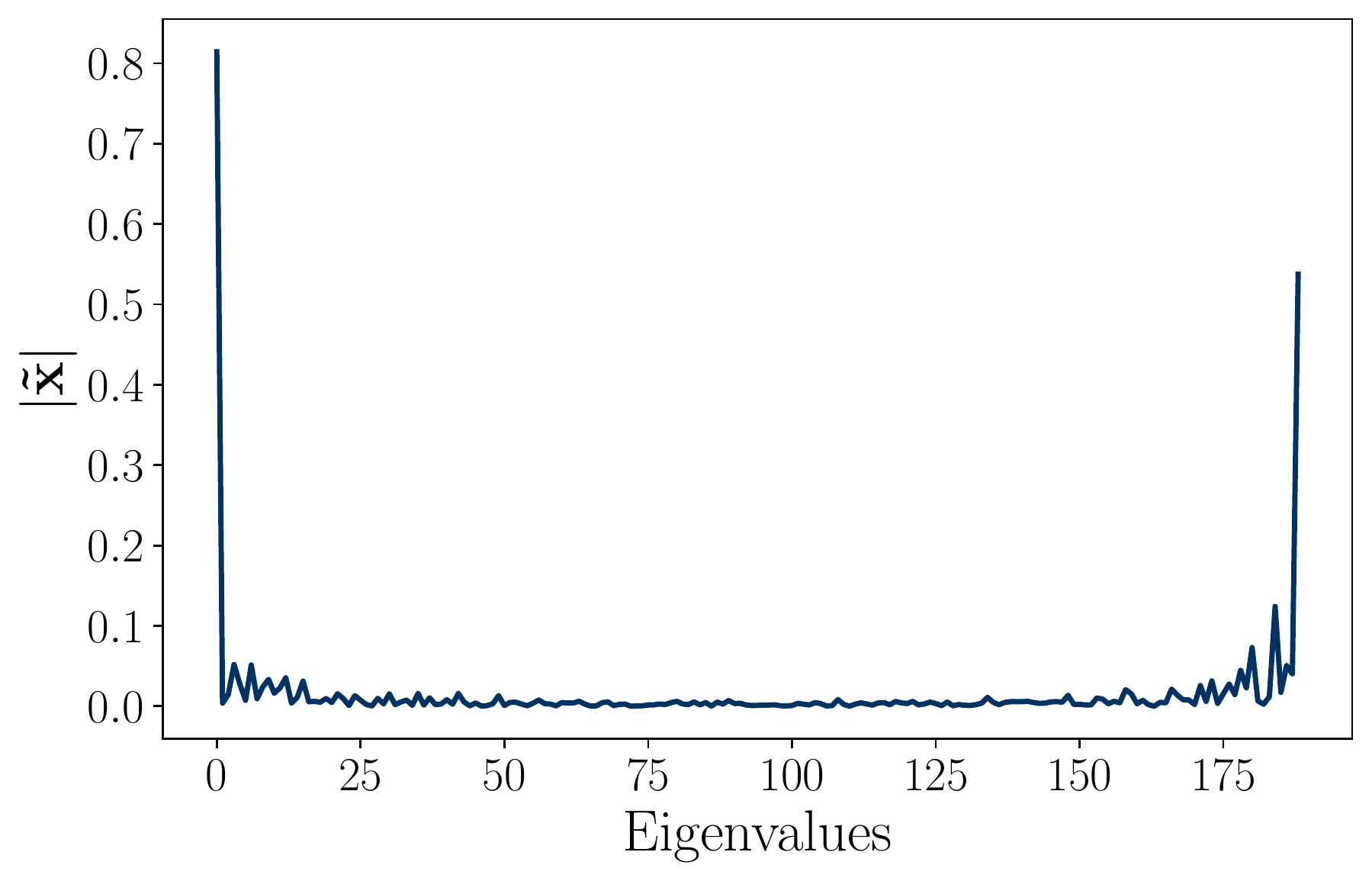}
    }
    \caption{Authorship attribution problem: Jane Austen. (\protect\subref{fig:austen:comparison}) Comparison between the Learn NVGF and three popular architectures (GCN \cite{Kipf2017-GCN}, SGC \cite{Weinberger2019-SGC} and GAT \cite{Velickovic2018-GAT}). (\protect\subref{fig:austen:change}) Change with respect to a GCNN \eqref{eq:GCNN} when considering no activation function (LSIGF), a designed NVGF \eqref{eq:designNVGF} and a learned NVGF \eqref{eq:learnNVGF}. (\protect\subref{fig:austen:input}) A sample of the input frequencies to the architectures, where a large high-frequency component is observed.}
    \label{fig:austen:results}
\end{figure*}

A graph convolutional neural network (GCNN) $\fnPhi$ has a layered architecture \cite{Defferrard2016-ChebNets, Gama2019-Archit}, where each layer applies a LSIGF as in \eqref{eq:LSIfilter}, followed by a pointwise nonlinear activation function $\rho:\fdR \to \fdR$ applied to each node $[\rho(\vcx)]_{i} = \rho([\vcx]_{i})$ and
\begin{equation} \label{eq:GCNN}
    \vcx_{0} = \vcx, \ \vcx_{\ell} = \rho \big( \fnH_{\ell}^{\text{lsi}} (\vcx_{\ell-1} ; \mtS, \vch_{\ell}) \big), \ \vcx_{L} = \fnPhi(\vcx;\mtS,\stH).
\end{equation}
The LSIGF $\fnH_{\ell}^{\text{lsi}}(\cdot;\mtS,\vch_{\ell})$ is characterized by the specific filter taps $\vch_{\ell} \in \fdR^{K_{\ell}+1}$. Note that the output of the GCNN is collected as the output of the last layer $\vcx_{L} = \fnPhi(\vcx;\mtS,\stH)$. This notation emphasizes that the input is $\vcx$, while the matrix $\mtS$ is given by the problem, and the filter taps $\stH = \{\vch_{\ell}\}_{\ell=1}^L$ are learned from data.

Nonlinear activation functions are used in GCNNs to enable them to learn nonlinear relationships between input and output. Additionally, theoretical results have found that they play a key role in creating frequency content that can be better processed by subsequent graph convolutional filters. As previously discussed, NVGFs are also capable of creating new frequency content, albeit in a linear manner. Therefore, by replacing the nonlinear activation functions by NVGFs, it is possible to decouple the contribution made by frequency creation from that made by the architecture's nonlinearity.

The first architecture proposed here is to use a designed NVGF in lieu of the activation function. That is, instead of using the nonlinear activation function $\rho$, an optimal NVGF filter designed as in Proposition~\ref{prop:NVGFoptimal} is used. This requires estimating the first and second moments of the NVGF input data. The architecture, herein termed ``Design NVGF'', is given by
\begin{equation} \label{eq:designNVGF}
    \vcx_{\ell} = \fnH_{\ell}^{\text{nv}} \big( \fnH_{\ell}^{\text{lsi}} (\vcx_{\ell}; \mtS, \vch_{\ell}); \mtS, \mtH_{\ell}^{\opt} \big).
\end{equation}
Note that, in this case, the filter taps $\mtH_{\ell}^{\opt} \in \fdR^{N \times (K_{\ell}+1)}$ of the NVGF are obtained by Proposition~\ref{prop:NVGFoptimal}, while the filter taps $\stH = \{\vch_{\ell}\}_{\ell=1}^L$ of the LSIGF are learned from data.


Alternatively, the filter taps of the NVGF replacing the nonlinear activation function can also be learned from data, together with the filter taps of the LSIGF:
\begin{equation} \label{eq:learnNVGF}
    \vcx_{\ell} = \fnH_{\ell}^{\text{nv}} \big( \fnH_{\ell}^{\text{lsi}} (\vcx_{\ell}; \mtS, \vch_{\ell}); \mtS, \mtH_{\ell} \big)
\end{equation}
where the filter taps to be learned are $\stH = \{(\vch_{\ell},\mtH_{\ell})\}_{\ell=1}^L$. This approach avoids the need to estimate first and second moments. Additionally, it allows the NVGF to learn how to create frequency content tailored to the application at hand, instead of just approximating the chosen nonlinear activation function. We note that while the increased number of parameters may lead to overfitting, this can be tackled by dropout. This architecture is termed ``Learn NVGF''.


%% file: 06-simsNVGF.tex


\ifundefined{arXiv}
The objective of the numerical experiment is to isolate the impact that frequency creation has on the overall performance. To do this, we focus on the GSP problem of authorship attribution, which is a problem with signals known to have high-frequency content \cite{Segarra2015-Authorship}.
\else
The objectives of the experiments are twofold. First, they aim to show how the architecture obtained by replacing the nonlinear activation function with a NVGF performs when compared to other popular GNNs. The second objective is to isolate the impact that frequency creation has on the overall performance. Due to space constraints, we focus on the problem of authorship attribution \cite{Segarra2015-Authorship}. Other applications can be found in the supplementary material, where the observations are summarized at the end of this section for reference.
\fi

\textbf{Problem statement.} In the problem of authorship attribution, the goal is to determine whether a given text has been written by a certain author, relying on other texts that the author is known to have written. To this end, word adjacency networks (WANs) are leveraged. WANs are graphs that are built by considering function words (i.e., words without semantic meaning) as nodes, and considering their co-occurrences as edges \cite{Segarra2015-Authorship}. As it happens, the way each author uses function words gives rise to a stylistic signature that can be used to attribute authorship. In what follows, we address this problem by leveraging previously constructed WANs as graphs, and the frequency count (histogram) of function words as the corresponding graph signals.

\textbf{Dataset.} For illustrative purposes, in what follows, works by Jane Austen are considered.
Attribution of other $\text{19}^{\text{th}}$ century authors can be found in the supplementary material.
A WAN consisting of $189$ nodes (function words) and $9812$ edges is built from texts belonging to a given corpus considered to be the training set. These texts are partitioned into segments of approximately $1000$ words each, and the frequency count of those $189$ function words in each of the texts is obtained. These represent the graph signals $\vcx$ that are considered to be part of the training set. Each of these is assigned a label $1$ to indicate that they have been authored by Jane Austen. An equal number of segments from other contemporary authors are randomly selected, and then their frequency count is computed and added to the training set with the label $0$ to indicate that they have not been written by Jane Austen. The total number of labeled samples in the training set is $1464$, of which $118$ are left aside for validation. The test set is built analogously by considering other text segments that were not part of the training set (and thus, not used to build the WAN either), totaling $78$ graph signals (half corresponding to texts authored by Jane Austen, and half corresponding to texts by other contemporary authors).

\textbf{Architectures and training.} For the first experiment, we compare the Learn NVGF architecture \eqref{eq:learnNVGF} with arguably three of the most popular non-GSP GNNs, namely, GCN \cite{Kipf2017-GCN}, SGC \cite{Weinberger2019-SGC}, and GAT \cite{Velickovic2018-GAT}. Note that the Learn NVGF is an entirely linear architecture, but one capable of creating frequencies due to the nature of the NVGF. The filter taps of both the LSIGF and the NVGF are learned from data. The other three architectures are nonlinear since they include a ReLU activation function after the first filtering layer. All architectures include a learnable linear readout layer. Dropout with a probability $0.5$ is included after the first layer. All architectures are trained for $25$ epochs with a batch size of $20$, using the ADAM algorithm \cite{Kingma15-ADAM} with the forgetting factors $0.9$ and $0.999$, respectively, as well as a learning rate $\eta$. The value of the learning rate $\eta$, the number of hidden units $F$, and the number of filter taps $K$ are chosen by exhaustive search over triplets $(\eta, F, K)$ in the set $\{0.001, 0.005, 0.01\} \times \{16, 32, 64\} \times \{2, 3, 4\}$.

\begin{figure*}
    \centering
    \subfloat[LSIGF]{%
        \label{subfig:austen:lsigf}
        \includegraphics[width=0.55\columnwidth]{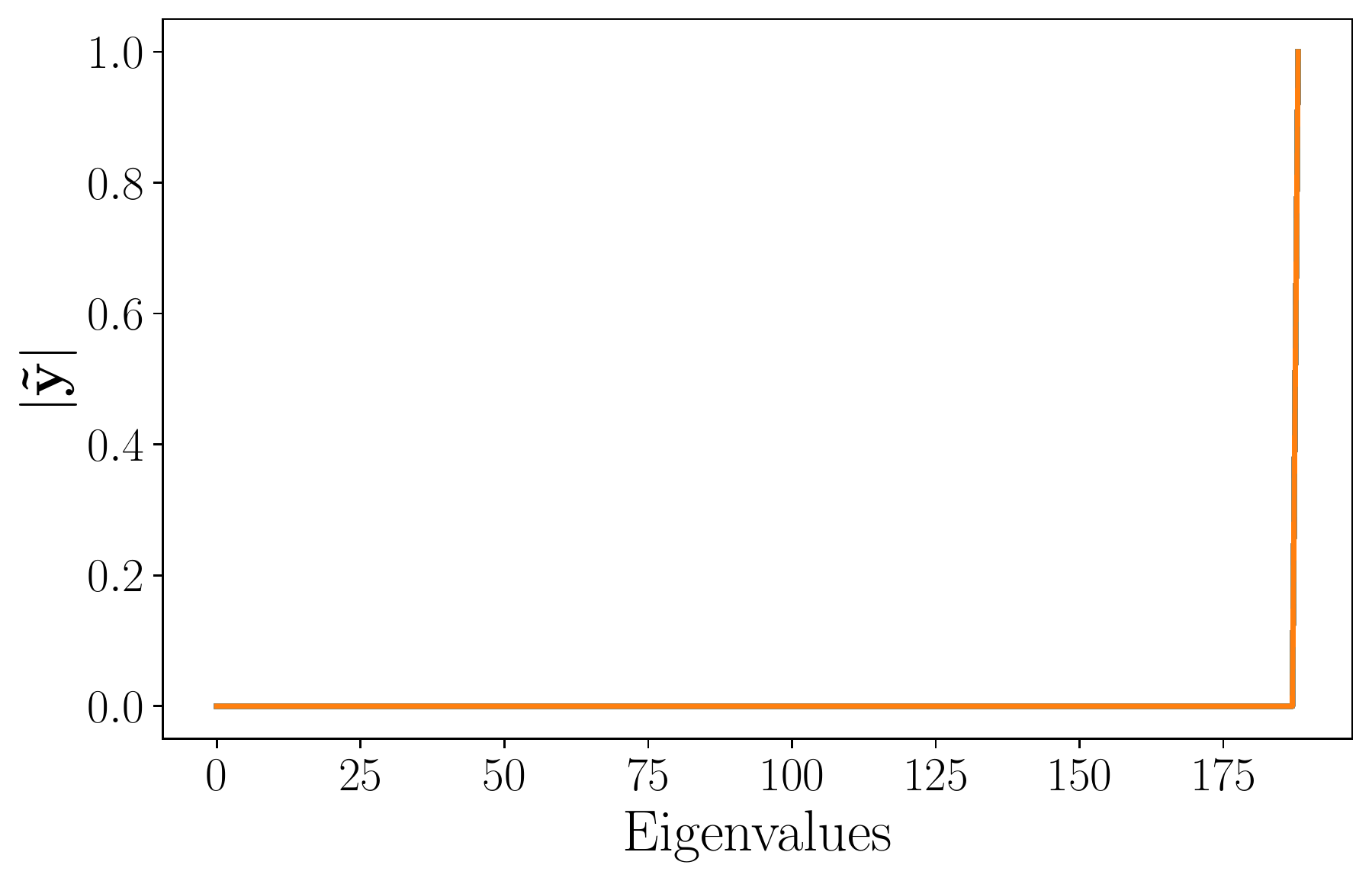}
    }
    \hfill
    \subfloat[GCNN]{%
        \label{subfig:austen:gcnn}
        \includegraphics[width=0.55\columnwidth]{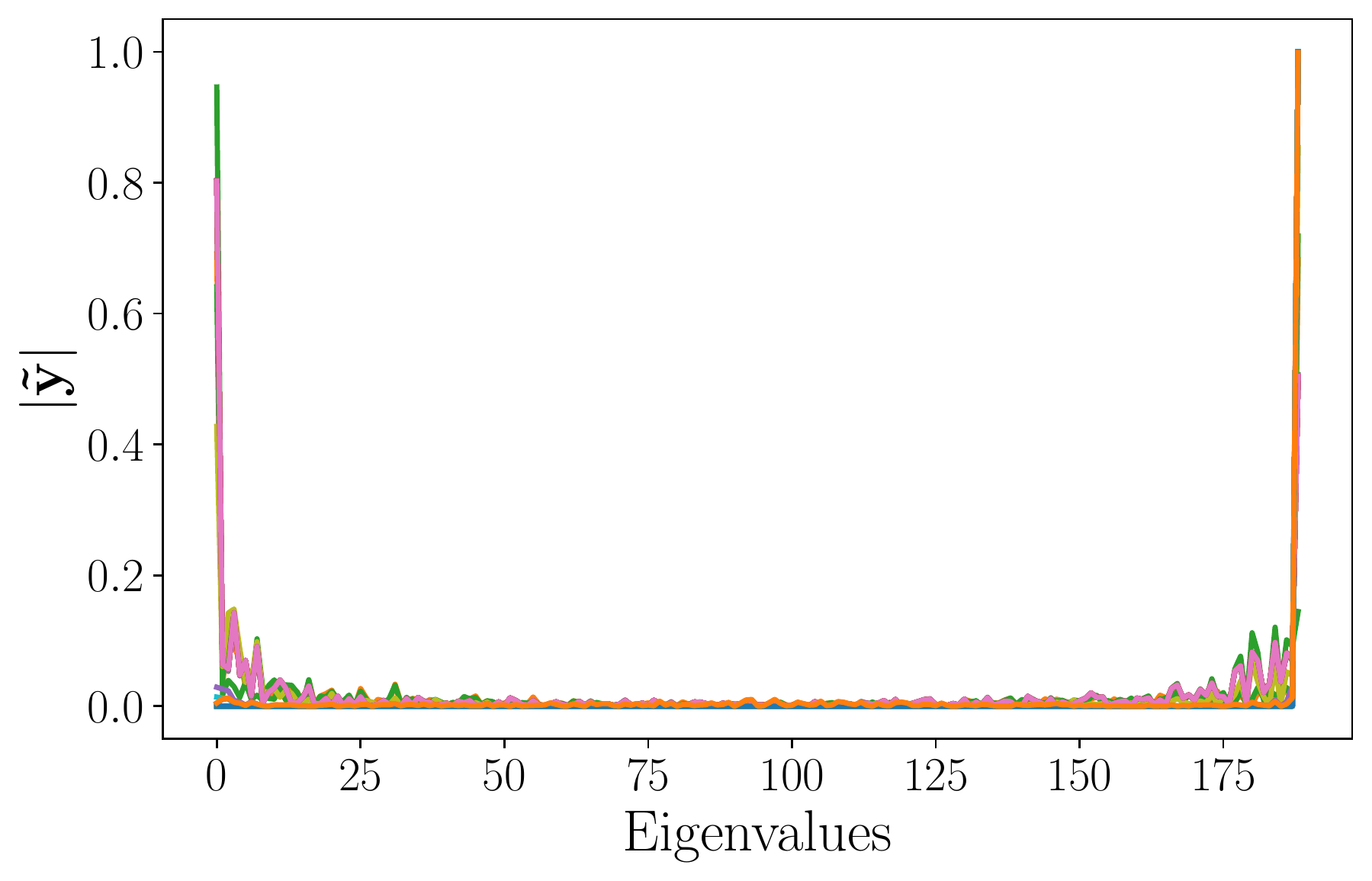}
    }
    \hfill
    \subfloat[Learn NVGF]{%
        \label{subfig:austen:nvgf}
        \includegraphics[width=0.55\columnwidth]{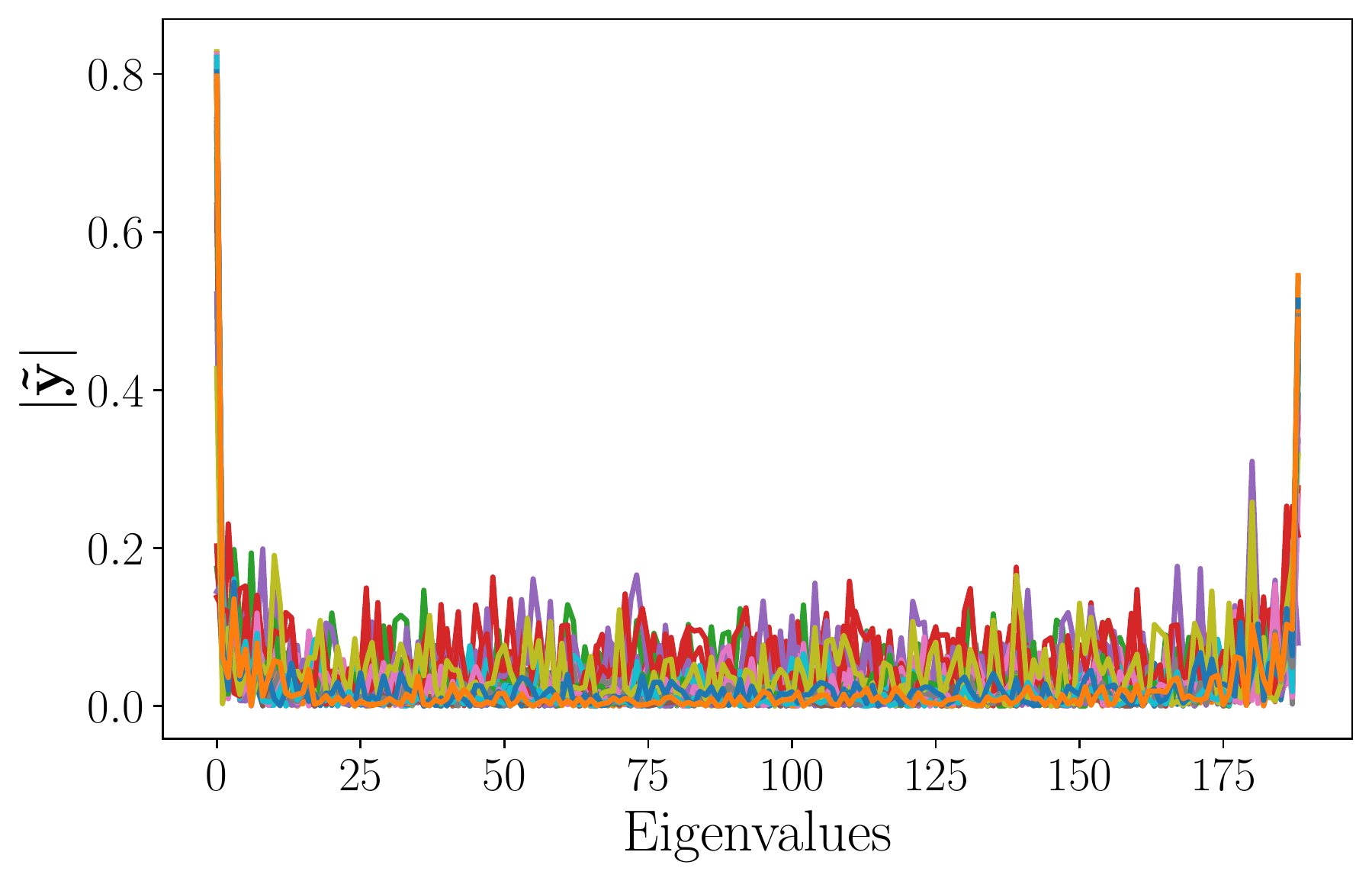}
    }
    \caption{Frequency responses to a single high-frequency input $\vcx = \vcv_{N}$. The frequency responses include the outputs of all $F=64$ filters.
    }
    \label{fig:austen:freq}
\end{figure*}

\textbf{Experiment 1: Performance comparison.} The objective of this first experiment is to illustrate that the performance of the Learn NVGF is comparable to the performance of popular (non-GSP) GNNs. The best results for each architecture are shown in Figure~\ref{fig:austen:comparison}, where the classification error was averaged over $10$ random splits of texts that are assigned to the training and test sets. One third of the standard deviation is also shown. It is observed that the Learn NVGF architecture has a comparable performance. It is emphasized that the objective of this paper is not to achieve state-of-the-art performance, but to offer insight on the role of nonlinear activation functions in frequency creation and how this translates to performance, as discussed next.
\ifundefined{arXiv}
\else
Among the three popular architectures, GAT ($\eta = 0.01$, $F=64$) exhibits the best performance, which is $102\%$ better than SGC ($\eta = 0.005$, $F=64$, $K=2$) and $30\%$ better than GCN ($\eta = 0.01$, $F=64$). This is expected since the graph filters involved in GAT are non-convolutional, and thus have more expressivity. Yet, when compared to the linear Learn NVGF ($\eta=0.001$, $F=32$, $K=3$), it is observed that the latter exhibits $20\%$ better performance than the GAT.
\fi

\textbf{Experiment 2: The role of frequency creation.} For the second experiment, we compare the Learn NVGF of the previous experiment with (i) a simple LSIGF, (ii) a Design NVGF as in \eqref{eq:designNVGF}, and (iii) a GCNN as in \eqref{eq:GCNN}. The same values of $(\eta=0.001, F=32, K=3)$ are used for all architectures as a means of fixing all other variabiliity except for the nonlinearity/frequency creation. The results are shown in Figure~\ref{fig:austen:change}. Note that the GCNN in \eqref{eq:GCNN} can be interpreted as a stand-in for ChebNets \cite{Defferrard2016-ChebNets}, arguably the most popular GSP-based GNN architecture.

\textbf{Discussion.} First, note that the LSIGF, Design NVGF, and Learn NVGF architectures are linear, whereas the GCNN is not. The LSIGF, however, is not capable of creating frequencies, while the other three are, albeit through different mechanisms. Second, Figure~\ref{fig:austen:change} shows the percentage difference in performance with respect to the base architecture (the GCNN). Essentially, the difference in performance between the LSIGF and the Learn NVGF can be pinned down to the frequency creation, because both architectures are linear, while the difference between the Learn NVGF and the GCNN can be tied to the nonlinear nature of the GCNN. It is thus observed that the LSIGF performs considerably worse than the Learn NVGF and the GCNN, which perform the same. It is observed that the Design NVGF performs halfway between the GCNN and the LSIGF. The Design NVGF depends on the ability to accurately estimate the first and second moments from the data, and this has an impact on its performance. In any case, it is noted that this experiment suggests that the main driver of improved performance is the frequency creation and not necessarily the nonlinear nature of the GCNN. 

From a qualitative standpoint, the average frequency response of the signals in the test set is shown in Figure~\ref{fig:austen:input}. Since the high-eigenvalue content is significant, it is expected that the ability to better process this content will impact the overall performance. This explains the relatively poor performance of the LSIGF. In Figure~\ref{fig:austen:freq}, we show the frequency response of the output for each of the three architectures (LSIGF, Learn NVGF, and GCNN) when the input has a single high frequency, i.e., $\vcx = \vcv_{N}$ so that $\vctx = \vce_{N}$. Figure~\ref{subfig:austen:lsigf} shows that the output frequency response of the LSIGF exhibits a single frequency, the same as the input. Figure~\ref{subfig:austen:gcnn} shows that the output frequency response of the GCNN has content in all frequencies, but most notably a low-frequency peak appears. Figure~\ref{subfig:austen:nvgf} shows that the output frequency response of the Learn NVGF contains all frequencies, in a much more spread manner than the GCNN.

\textbf{General observations.} In the supplementary material, a similar analysis is carried out for $21$ other authors. Additionally, it is noted that Jane Austen is representative of the largest group (consisting of $11$ authors) where the Learn NVGF and the GCNN have similar performance and are better than the LSIGF. For $7$ other authors, the Learn NVGF actually performs better than the GCNN. Finally, for the remaining $3$ authors, there is no significant difference between the performance of the LSIGF, the GCNN, and the Learn NVGF, which implies that the high-frequency eigenvalue content is less significant for these authors. \ifundefined{arXiv} \else Additionally, the problem of movie recommendation is considered \cite{Harper2016-MovieLens}. First, it is observed that the Learn NVGF exhibits better performance than the methods in \cite{Monti2017-RecommendationGNN, Levie2018-CayleyNets} and the nearest neighbor algorithm. Second, it is noted that the input signals do not have significant high-eigenvalue content and thus the LSIGF, the GCNN, and the Learn NVGF exhibit similar performance. For a detailed analysis, please see the supplementary material. \footnote{It is noted that the popular benchmark of semi-supervised learning over citation networks does not fit the empirical risk minimization framework nor the graph signal processing framework, and thus, it does not admit a frequency analysis, precluding their use in this work.}

\textbf{Problems beyond processing graph signals.} Two popular tasks that utilize graph-based data are semi-supervised learning and graph classification. The former involves a framework where each sample represents a node in the graph, i.e., $\vcx \in \fdR^{F}$ for $F$ features instead of $\vcx \in \fdR^{N}$ for $N$ nodes. This implies that the notion of graph frequency response used in \eqref{eq:GFT} does not translate to the semi-supervised setting. Therefore, extending these results to this problem requires careful determination of the notion of frequency. For the graph classification problem, each sample in the dataset represents a signal together with a graph, and the graph associated to each sample is usually different. In this case, while the notion of frequency is properly defined for each sample, an overarching frequency response for all the samples is not. Possible extensions of this framework to the graph classification problem would entail redefining a common notion of frequency response for the dataset, for example, using graphons.
\fi

%% file: 07-conclusionsNVGF.tex


The objective of this work was to study the role of frequency creation in GSP problems. To do so, nonlinear activation functions (which theoretical findings suggest give rise to frequency creation) are replaced by NVGFs, which are also capable of creating frequencies, but in a linear manner. In this way, frequency creation was decoupled from the nonlinear nature of activation functions. Numerical experiments show that the main driver of improved performance is frequency creation and not necessarily the nonlinear nature of GCNNs. As future work, we are interested in extending this frequency analysis to non-GSP related problems such as semi-supervised node classification or graph classification problems, which require a careful definition of a notion of frequency.
\ifundefined{arXiv}
\else We discussed the caveats of extending this framework to include semi-supervised learning and graph classification problems, relating to the need of defining an appropriate notion of graph frequency. This opens up an interesting area of future research. We also note that it would be possible to use shift-variant filters to leverage this framework when using CNNs.\fi

%% file: AC-appendixShort.tex

In what follows, we present additional simulations on the authorship attribution problem. The proofs, as well as further additional simulations on the problem of movie recommendations, can be found online\footnote{\url{https://anonymous.4open.science/r/nvgf-B8B2/}}.


\begin{figure*}
    \centering
    \includegraphics[width=0.9\textwidth]{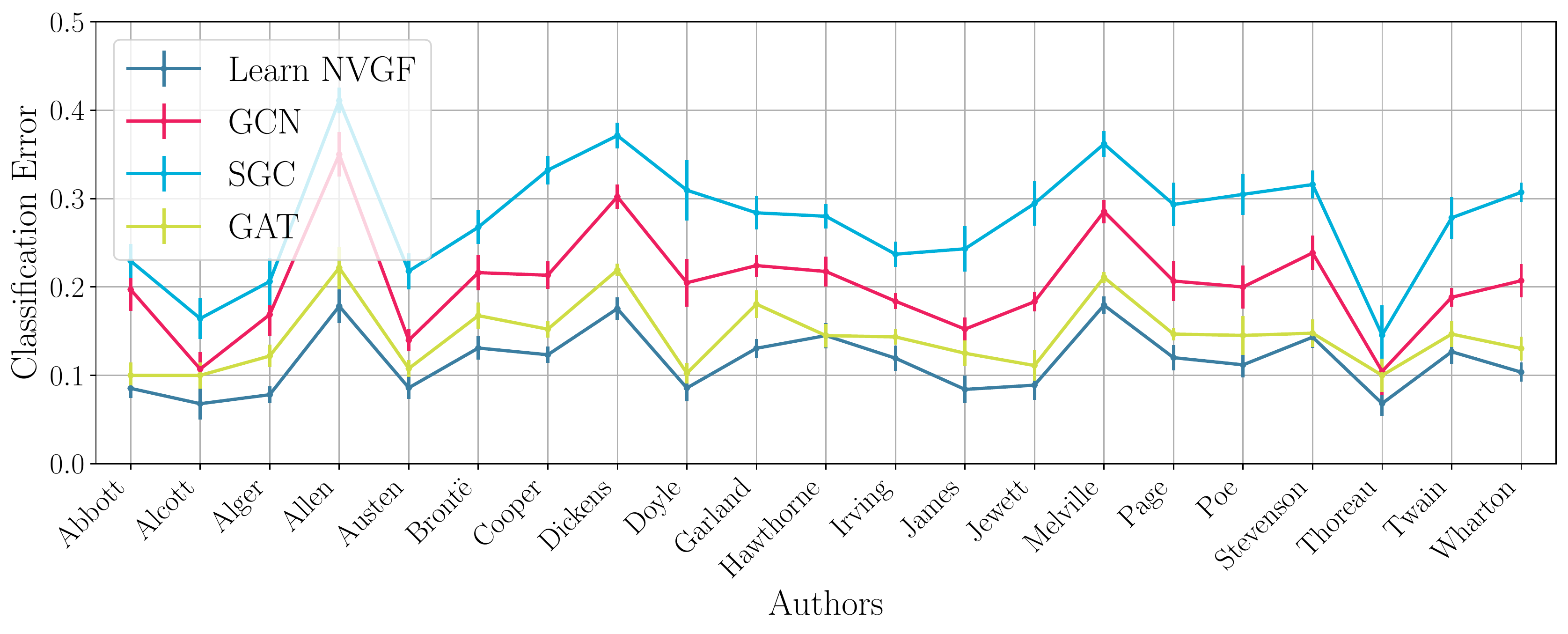}
    \caption{Comparison between architectures for all authors. It is observed that, in most cases, the Learn NVGF exhibits significantly better performance than the GAT \cite{Velickovic2018-GAT}, which in turn is better than the GCN \cite{Kipf2017-GCN}, and all of them are better than the SGC \cite{Weinberger2019-SGC}. The error bars reflect one third of the estimated standard deviation.}
    \label{fig:app:allAuthors:compare}
\end{figure*}

\textbf{Problem objective.} Consider a set $\stTheta_{a}$ of texts that are known to be written by author $a$. Given a new text $\theta \notin \stTheta_{a}$ the objective is to determine whether $\theta$ was written by $a$ or not.

\textbf{Approach.} The approach is to leverage word adjacency networks (WAN) built from the set of known texts $\stTheta_{a}$ to build a graph support $\stG$, and then use the word frequency count of the function words in $\theta$ as the graph signal $\vcx$. Then, $\vcx$ is processed through a GNN $\fnPhi$ (in any of the variants discussed in Section~\ref{sec:archit}) to obtain a predictor of the text being written by author $a$. See \cite{Segarra2015-Authorship} for details.

\textbf{Dataset and code.} The dataset is presented in \cite{Segarra2015-Authorship}, and is publicly available at \url{http://github.com/alelab-upenn/graph-neural-networks/tree/master/datasets/authorshipData}. The dataset consists of $21$ authors from the $\text{19}^{\text{th}}$ century. The corpus of each author is split in texts of about $1000$ words. The WANs and the word frequency count for each text are already present in the dataset. The texts are split, at random, in $95\%$ for training and $5\%$ for testing. The weights of the WANs of the $95\%$ texts selected for training are averaged to obtain an average WAN from which the GMD $\mtS$ is obtained. The resulting support matrix is further normalized to have unit spectral norm. Note that no text from the test set is used in building the WAN. From the texts in the training set, $8\%$ are further separated to build the validation set. Denote by $N_{a}^{\text{train}}$, $N_{a}^{\text{valid}}$, and $N_{a}^{\text{test}}$ the number of texts in the training, validation, and test set, respectively.  The word frequency counts for each text are normalized and used as graph signals. A label of $1$ is attached to these signals to indicate that they have been written by author $a$ in a supervised learning context. To complete the datasets, an equal number of texts are obtained at random from other contemporary authors, and their frequency word counts are normalized and incorporated into the corresponding sets and assigned a label of $0$. In this way, the resulting training, validation, and test set have $2N_{a}^{\text{train}}$, $2N_{a}^{\text{valid}}$, and $2N_{a}^{\text{test}}$ samples, respectively (half of them labeled with $1$ and the other half with $0$). The code to run the simulations can be found online.

\textbf{Training.} The loss function during training is a cross-entropy loss between the logits obtained from the output of each architecture, and the labels in the training set. All the architectures are trained by using an ADAM optimizer \cite{Kingma15-ADAM} with forgetting factors $0.9$ and $0.999$, and with a learning rate $\eta$. The training is carried out for $25$ epochs with batches of size $20$. Dropout with probability $0.5$ is included before the readout layer, during training, to avoid overfitting. At testing time, the weights are correspondingly rescaled. Validation is run every $5$ training steps. The learned filters that result in the best performance on the validation set are kept and used during the testing phase. For each experiment, $10$ realizations of the random train/test split are carried out (also randomizing the selection of the texts by other authors that complete the training, validation, and test sets). The average evaluation performances (measured as classification error---ratio of texts wrongly attributed in the test set) is reported, together with the estimated standard deviation.

\textbf{Hyperparameter selection.} The number of hidden units $F_{\text{NVGF}}$, $F_{\text{GCN}}$, $F_{\text{SGC}}$, and $F_{\text{GAT}}$, the polynomial order $K_{\text{NVGF}}$ and $K_{\text{SGC}}$, and the learning rate $\eta_{\text{NVGF}}$, $\eta_{\text{GCN}}$, $\eta_{\text{SGC}}$, and $\eta_{\text{GAT}}$ are selected, independently for each architecture, from the set $\{16,32,64\}$ for the number of hidden units, $\{2, 3, 4\}$ for the polynomial order, and $\{0.001, 0.005, 0.01\}$ for the learning rate. In other words, all possible combinations of these three parameters are run for each architecture, and the ones that show the best performance on the test set are kept.

\begin{figure*}
    \centering
    \includegraphics[width=0.9\textwidth]{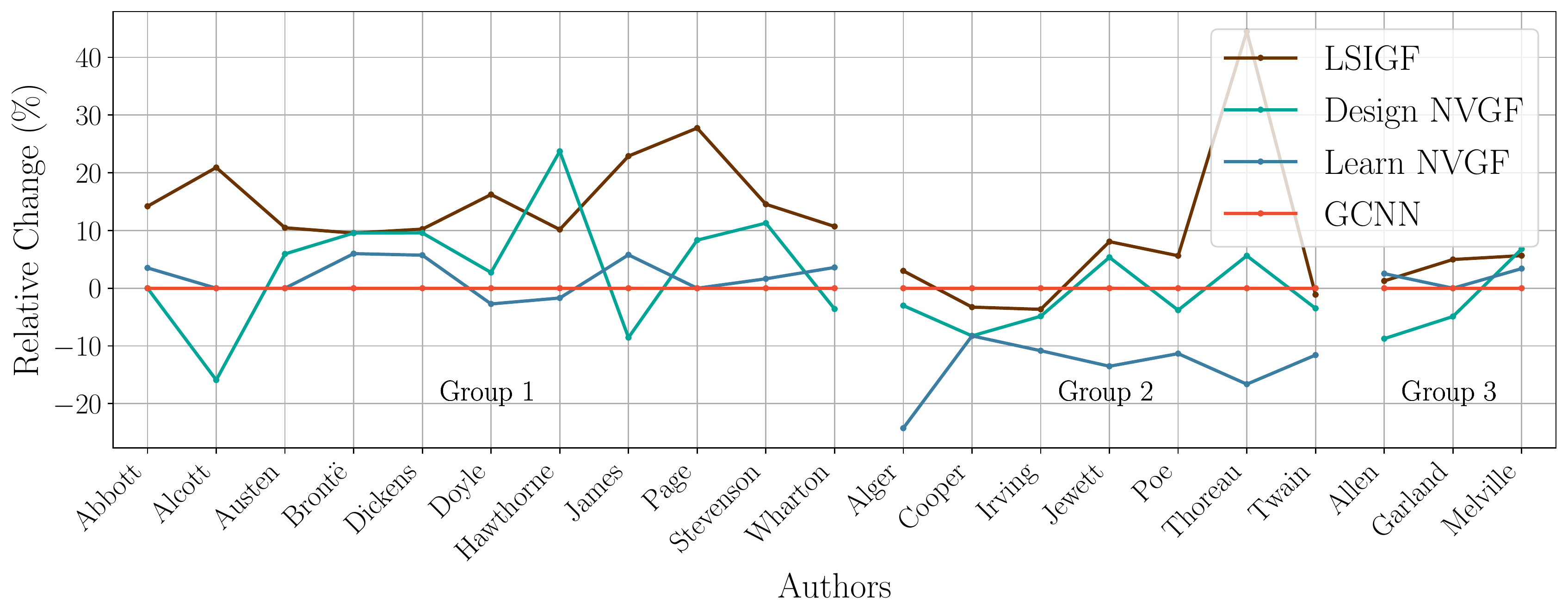}
    \caption{Relative change in performance with respect to the GCNN baseline, divided into $3$ groups of similar behavior. Group 1 ($53\%$ of the authors) includes those where the Learn NVGF has a comparable performance to the GCNN and both of them are better than the LSIGF, showing that frequency creation plays a vital role in improving performance. Group 2 ($33\%$ of the authors) shows that oftentimes, the Learn NVGF can improve significantly over the GCNN, suggesting that the nonlinear nature of the mapping may have a negative impact. Group 3 ($14\%$ of the authors) consists of those cases when the Learn NVGF, the GCNN, and the LSIGF all exhibit comparable performance.}
    \label{fig:app:allAuthors:change}
\end{figure*}

\textbf{Experiment 1: Performance comparison.} In the first experiment, the performance is measured by classification error (ratio of texts in the test set that were wrongly attributed), and the comparison between the Learn NVGF, the GCN \cite{Kipf2017-GCN}, the SGC \cite{Weinberger2019-SGC}, and the GAT \cite{Velickovic2018-GAT} is carried out, for all $21$ authors. The hyperparameters used for each architecture and each author where selected as those that offer the best performance, and can be found online.

Results are shown in Figure~\ref{fig:app:allAuthors:compare}. The general trend observed is for the Learn NVGF to exhibit better performance than the GAT, which in turn is better than the GCN, and all of them are better than SGC. The differences are usually significant between the four architectures, with a marked improvement by the Learn NVGF. It is observed, however, that for Doyle, Hawthorne, and Stevenson, the performance of the Learn NVGF and the GAT is comparable. In any case, it is emphasized that the goal of this experiment is not to achieve state-of-the-art performance, but to show that the performance is comparable to the most popular GNN architectures. The objective of this work is to analyze the role of frequency creation, and decouple it from the effect of nonlinear activation functions, as discussed next.

\textbf{Experiment 2: Impact of nonlinearities.} In the second experiment, the objective is to decouple the contribution made by frequency creation from that made by the nonlinear nature of the architecture. To do this, the GCNN architecture is taken as a baseline (a nonlinear, frequency-creating architecture), and the relative change in performance of the three other architectures is measured (Learn NVGF and Design NVGF, both linear and frequency-creating architectures, and LSIGF, which is linear but cannot create frequencies). The learning rate $\eta$, the number of hidden units $F$, and order of the filters $K$ are the same for all four architectures. The results showing relative change in the mean classification error with respect to the GCNN are shown in Figure~\ref{fig:app:allAuthors:change}.

The authors have been classified into three groups according to their relative behavior. Group 1 consists of those authors where the Learn NVGF has a comparable performance with respect to the GCNN (i.e., $5\%$ or less relative variation in performance), and both have a considerably better performance than the LSIGF. These results suggest that, for this group of authors (the most numerous one, consisting of $11$ authors, or $53\%$ of the total), the improvement in performance is mostly due to the capability of the architectures to create frequency, and not necessarily due to the nonlinear nature of the GCNN.

Group 2 consists of those authors where the Learn NVGF exhibits a better performance than the GCNN, which in turn exhibits a similar performance to the LSIGF (except for Jewett and Thoreau, where the GCNN still exhibits considerably better performance than the LSIGF). This group of 7 authors ($33\%$ of the total) suggests that in some cases, frequency creation is the key contributor to improved performance, and that the inclusion of a nonlinear mapping could possibly have a negative impact. In essence, we observe that a linear architecture capable of creating frequencies outperforms the rest, and that a frequency-creating nonlinear architecture performs similarly to a linear architecture that cannot create frequencies. This suggests that the relationship between input and output is approximately linear with frequencies being created, but that attempting to model this frequency creation with a nonlinear architecture is not a good approach.

Finally, Group 3 consists of those architectures for which all architectures exhibit a similar performance. This group consists of $3$ authors ($14\%$ of the total). This may be explained by the fact that the high-frequency content for these authors does not carry useful information, and thus the role of frequency creation is less relevant.

With respect to the Design NVGF, a somewhat erratic behavior is observed. In near half of the cases ($11$ authors), the performance of the Design NVGF closely resembles the performance of the GCNN, as expected. In a few other cases (Alcott, James, Cooper, and Allen), the Design NVGF results are better, and in the rest they are considerably worse (Bront\"{e}, Dickens, Hawthorne, Page, Stevenson, and Melville). The Design NVGF architecture is designed to mimic the GCNN, but its accurate design depends on good estimates of the first and second moments of the data. Thus, one possible explanation is that there is not enough data to get good estimates of these values. Another possible explanation is that higher-order moments have a larger impact in these cases, and the NVGF, being linear, is not able to accurately capture them.

Overall, this second experiment shows the importance of frequency-creation in improving performance, especially when high-frequency content is significant. Among the two ways of creating frequency (linear and nonlinear), we see that in most cases, they essentially perform the same. But there are cases in which creating frequency content in a linear manner is better (Group 2). In any case, this experiment shows the importance of frequency creation and calls for further research on what other contributions the nonlinear nature of the activation function has on performance.

%% file: AA-appendixProofs.tex



\section{Frequency Response of Node-Variant Graph Filters}

\begin{proof}[Proof of Proposition~\ref{prop:GFToutputNVGF}]
The output graph signal $\vcy \in \fdR^{N}$ of a node-variant graph filter is given by \eqref{eq:NVGF}, which is reproduced here for ease of exposition:
\begin{equation} \label{eq:app:NVGF}
    \vcy = \sum_{k=0}^{K} \diag(\vch^{(k)}) \mtS^{k} \vcx.
\end{equation}
Recall that $\vch^{(k)} \in \fdR^{N}$ is the $(k+1)^{\text{th}}$ column of the matrix $\mtH \in \fdR^{N \times (K+1)}$ containing the $N(K+1)$ filter taps, that $\mtS \in \fdR^{N \times N}$ is the support matrix describing the graph, and that $\vcx \in \fdR^{N}$ is the input graph signal. As given by \eqref{eq:GFT}, the spectral representation $\vctx \in \fdR^{N}$ of the input graph signal are the coordinates of representing $\vcx$ on the eigenbasis $\{\vcv_{i}\}_{i=1}^N$ of the support matrix $\mtS$, i.e., $\vctx = \mtV^{\Tr} \vcx$. Similarly, $\vcty = \mtV^{\Tr} \vcy$. Therefore, by \eqref{eq:app:NVGF} together with the fact that $\mtS  = \mtV \diag(\vclambda) \mtV^{\Tr}$, it holds that
\begin{equation} \label{eq:app:NVGFoutputGFT}
\begin{aligned}
    \vcty & = \mtV^{\Tr} \vcy = \mtV^{\Tr} \sum_{k=0}^{K} \diag(\vch^{(k)}) \mtS^{k} \vcx \\ & = \mtV^{\Tr} \sum_{k=0}^{K} \diag(\vch^{(k)}) \mtV \diag(\vclambda^{k}) \vctx
\end{aligned}
\end{equation}
where $\vclambda^{k}$ is the $N$-vector with $i^\text{th}$ element given by $[\vclambda^{k}]_{i} = \lambda_{i}^{k}$. Denoting $h_{ik} = [\vch^{(k)}]_{i} = [\mtH]_{i(k+1)}$, note that $\diag(\vch^{(k)})\mtV \diag(\vclambda^{k}) \in \fdR^{N \times N}$, so that
\begin{equation}  \label{eq:app:diagVdiag}
\begin{aligned}
    \sum_{k=0}^{K} [\diag(& \vch^{(k)})\mtV \diag(\vclambda^{k})]_{ij}  = \sum_{k=0}^{K} h_{ik} \lambda_{j}^{k} v_{ij} \\ & = v_{ij} \sum_{k=0}^{K} h_{ik} \lambda_{j}^{k} = v_{ij} \fnth_{i}(\lambda_{j})
\end{aligned}
\end{equation}
where $\fnth_{i}(\lambda) = \sum_{k=0}^{K} h_{ik} \lambda^{k}$ is the frequency response of the filter taps at node $\lmv_{i}$; see \eqref{eq:freqResponse}. Next, observe that the filter taps for each node are collected in the rows of $\mtH$. Therefore, in analogy to \eqref{eq:GFTfilter}, it holds that $\mtH \mtLambda^{\Tr} \in \fdR^{N \times N}$, with
\begin{equation} \label{eq:app:LambdaH}
    [\mtH \mtLambda^{\Tr}]_{ij} = \sum_{k=0}^{K} h_{ik} \lambda_{j}^{k} = \fnth_{i} (\lambda_{j}).
\end{equation}
Substituting \eqref{eq:app:LambdaH} into \eqref{eq:app:diagVdiag} gives that
\begin{equation}
    \sum_{k=0}^{K} [\diag(\vch^{(k)})\mtV \diag(\vclambda^{k})]_{ij} = [\mtV \circ (\mtH \mtLambda^{\Tr})]_{ij}
\end{equation}
so that \eqref{eq:app:NVGFoutputGFT} becomes
\begin{equation}
    \vcty = \mtV^{\Tr} \big( \mtV \circ (\mtH \mtLambda^{\Tr}) \big) \vctx.
\end{equation}
This completes the proof.
\end{proof}

\begin{proof}[Proof of Corollary \ref{cor:NVGFnewFreq}]
The $j^{\text{th}}$ element of the graph Fourier transform of the output, $\vcty$, is
\begin{equation} \label{eq:app:gftOutput}
    \scty_{j}
    = \sum_{i=1}^{N} \sctx_{i} \vcv_{j}^{\Tr} \diag(\vcfnth(\lambda_{i})) \vcv_{i}
\end{equation}
where $\vcfnth(\lambda_i) = [\fnth_{1}(\lambda_i),\fnth_{2}(\lambda_i),\dots,\fnth_{N}(\lambda_i)] \in \fdR^{N}$ collects the frequency response of all nodes at eigenvalue $\lambda_{i}$. Note that, if $\scty_{j}$ depends on $\sctx_{i}$ for some $i \neq j$, then frequencies are created. When $\mtV^{\Tr} \big( \mtV \circ (\mtH \mtLambda^{\Tr}) \big)$ is not diagonal, there exists $i,j\in\{1,\ldots,N\}$ such that $i\ne j$ and $(\mtV^{\Tr} \big( \mtV \circ (\mtH \mtLambda^{\Tr}) \big))_{ij} = \vcv_{j}^{\Tr} \diag(\vcfnth(\lambda_{i})) \vcv_{i} \ne 0$. In this case, \eqref{eq:app:gftOutput} indeed shows that frequencies are created.

Now suppose that no frequency creation occurs. Then $\mtV^{\Tr} \big( \mtV \circ (\mtH \mtLambda^{\Tr}) \big)$ must be diagonal. Therefore, there exists $\alpha_1,\ldots,\alpha_N\in\mathbb{R}$ such that $\vcv_{j}^{\Tr} \diag(\vcfnth(\lambda_{i})) \vcv_{i} = \alpha_{i} \delta_{ij}$ for all $i,j$. Recalling that the set $\{\vcv_{i}\}_{i=1}^N$ of eigenvectors of $\mtS$ is an orthonormal basis, we have that $\vcv_{j}^{\Tr} \vcv_{i} = \delta_{ij}$, so it must be that $\vcv_{i}^{\Tr}(\diag(\vcfnth(\lambda_{i}))-\alpha_i \mtI) \vcv_{j} = 0$ for all $i,j$. This implies for all $i,j$ that $v_{ji}=0$ or $\fnth_{j}(\lambda_i)=\alpha_i$. In the case that $v_{ij}\ne 0$ for all $i,j$, it is clear that $\diag(\vcfnth(\lambda_{i})) = \alpha_{i} \mtI$, meaning that the frequency response at all nodes is the same. This restriction implies that the NVGF filter is a LSI graph filter. This concludes the proof.
\end{proof}


\section{Lipschitz Continuity of Node-Variant Graph Filters}

\begin{proof}[Proof of Theorem~\ref{thm:stability}]
Leveraging the fact that the filter taps of both $\mtfnH^{\text{nv}}(\mthS)$ and $\mtfnH^{\text{nv}}(\mtS)$ are the same, start by writing the difference between the filter outputs as
\begin{equation} \label{eq:app:outputDiff}
    \big( \mtfnH^{\text{nv}}(\mthS)-\mtfnH^{\text{nv}}(\mtS) \big) \vcx = \sum_{k=0}^{K} \diag(\vch^{(k)}) \big(\mthS^{k} - \mtS^{k}\big) \vcx.
\end{equation}
Let $\mtE = \mthS - \mtS$ and note that $\mtE$ is symmetric and satisfies $\| \mtE \| = \|\mthS - \mtS\| \leq \sceps$ by assumption. Recall that $(\mtS+\mtE)^{k} = \mtS^{k} + \sum_{r=0}^{k-1} \mtS^{r} \mtE \mtS^{k-r-1} + \mtC$ with $\mtC$ such that $\|\mtC\| \leq \sum_{r=2}^{k} \binom{k}{r} \|\mtE\|^{r}\|\mtS\|^{k-r}$. Leveraging this fact in \eqref{eq:app:outputDiff}, it yields
\begin{equation} \label{eq:app:outputDiffWithD}
\begin{aligned}
    \big( \mtfnH^{\text{nv}}& (\mthS)-\mtfnH^{\text{nv}}(\mtS) \big) \vcx  \\ & = \sum_{k=0}^{K} \diag(\vch^{(k)}) \sum_{r=0}^{k-1} \mtS^{r} \mtE \mtS^{k-r-1} \vcx + \mtD \vcx
\end{aligned}
\end{equation}
with $\mtD$ satisfying $\|\mtD\| = \bigOh(\|\mtE\|^{2})$, since the filter taps $\mtH$ define analytic frequency responses $\fnh_{i}(\lambda)$ with bounded derivatives for all $i \in \{1,\ldots,N\}$. The input graph signal $\vcx$ can be rewritten as $\vcx = \sum_{i=1}^{N} \sctx_{i} \vcv_{i}$ using the GFT for the support matrix $\mtS$. Then, \eqref{eq:app:outputDiffWithD} becomes
\begin{align} \label{eq:app:outputDiffGFT}
    & \big( \mtfnH^{\text{nv}}(\mthS)-\mtfnH^{\text{nv}}(\mtS) \big) \vcx \\ & = \sum_{i=1}^{N} \sctx_{i} \sum_{k=0}^{K} \diag(\vch^{(k)}) \sum_{r=0}^{k-1} \mtS^{r} \mtE \mtS^{k-r-1} \vcv_{i} + \sum_{i=1}^{N} \sctx_{i} \mtD \vcv_{i}. \nonumber
\end{align}

Using the fact that $\vcv_{i}$ is an eigenvector of $\mtS$ the first term in \eqref{eq:app:outputDiffGFT} can be conveniently rewritten as
\begin{equation} \label{eq:app:outputDiffGFTfirstTerm}
\begin{aligned}
    \sum_{i=1}^{N} \sctx_{i} & \sum_{k=0}^{K} \diag(\vch^{(k)}) \sum_{r=0}^{k-1} \mtS^{r} \mtE \mtS^{k-r-1} \vcv_{i} \\ & = \sum_{i=1}^{N} \sctx_{i} \sum_{k=0}^{K} \diag(\vch^{(k)}) \sum_{r=0}^{k-1}\lambda_{i}^{k-r-1}\mtS^{r} \mtE \vcv_{i}.
\end{aligned}
\end{equation}
Denoting by $\mtE = \mtU \diag(\vcm) \mtU^{\Tr}$ the eigendecomposition of $\mtE$ with $\mtE\vcu_{i} = m_{i}\vcu_{i}$ the corresponding eigenvectors $\vcu_{i}$ and eigenvalues $m_{i}$, \cite[Lemma 1]{Gama2020-Stability} states that $\mtE \vcv_{i} = m_{i} \vcv_{i} + \mtE_{U} \vcv_{i}$ with $\| \mtE_{U}\| \leq 8\sceps$ for an appropriate matrix $\mtE_{U}$ that depends on $\mtU$, $\mtS$, and $\mtE$. Using this in \eqref{eq:app:outputDiffGFTfirstTerm},
\begin{align}
    \sum_{i=1}^{N} & \sctx_{i} \sum_{k=0}^{K} \diag(\vch^{(k)}) \sum_{r=0}^{k-1}\lambda_{i}^{k-r-1}\mtS^{r} \mtE \vcv_{i} \nonumber \\ & =
    \sum_{i=1}^{N} \sctx_{i} m_{i} \sum_{k=0}^{K} \diag(\vch^{(k)}) \sum_{r=0}^{k-1}\lambda_{i}^{k-r-1}\mtS^{r} \vcv_{i} \label{eq:app:outputDiffGFTfirstTermFirst}\\
    & \quad + \sum_{i=1}^{N} \sctx_{i} \sum_{k=0}^{K} \diag(\vch^{(k)}) \sum_{r=0}^{k-1}\lambda_{i}^{k-r-1}\mtS^{r} \mtE_{U} \vcv_{i}\label{eq:app:outputDiffGFTfirstTermSecond}
\end{align}
is obtained. For the term \eqref{eq:app:outputDiffGFTfirstTermFirst}, note that $\mtS^{r} \vcv_{i} = \lambda_{i}^{r} \vcv_{i}$ and that $\sum_{r=0}^{k-1} \lambda_{i}^{k-r-1}\lambda_{i}^{r} = k \lambda_{i}^{k-1}$, so that the term \eqref{eq:app:outputDiffGFTfirstTermFirst} is equivalent to
\begin{equation} \label{eq:app:outputDiffGFTfirstTermFirstSolved}
\begin{aligned}
    \sum_{i=1}^{N} \sctx_{i} m_{i} & \sum_{k=0}^{K} k\diag(\vch^{(k)})\lambda_{i}^{k-1} \vcv_{i} \\
    & = \sum_{i=1}^{N} \sctx_{i} m_{i} \diag(\vcfnth'(\lambda_{i}))\vcv_{i}
\end{aligned}
\end{equation}
where $\vcfnth'(\lambda) \in \fdR^{N}$ is a vector where the $i^{\text{th}}$ entry is the derivative of the frequency response of node $\lmv_{i}$, i.e., $[\vcfnth'(\lambda)]_{i} = \fnth'_{i}(\lambda) = \frac{d}{d\lambda} \fnth_{i} (\lambda)$. To rewrite \eqref{eq:app:outputDiffGFTfirstTermSecond}, consider the following lemma, which is conveniently proved after completing the current proof.

\begin{lemma}\label{lem:app:GiMatrix}
For all $i\in\{1,\dots,N\}$, define $\mtG_{i} \in \fdR^{N \times N}$ by
\begin{equation}\label{eq:app:GiMatrix}
    [\mtG_{i}]_{tj} = \begin{cases}
        \fnth_{t}'(\lambda_{i}) & \text{ if }j=i \\
        \frac{\fnth_{t}(\lambda_{i})-\fnth_{t}(\lambda_{j})}{\lambda_{i}-\lambda_{j}} & \text{ if } j \neq i
    \end{cases}
\end{equation}
where $\fnth_{t}$ is the frequency response at node $\lmv_{t}$ and $\fnth'_{t}$ is its derivative. Then
\begin{equation} \label{eq:app:outputDiffGFTfirstTermSecondSolved}
\begin{aligned}
    \sum_{i=1}^{N} \sctx_{i} & \sum_{k=0}^{K} \diag(\vch^{(k)}) \sum_{r=0}^{k-1}\lambda_{i}^{k-r-1}\mtS^{r} \mtE_{U} \vcv_{i} \\ & = \sum_{i=1}^{N} \sctx_{i} \big( \mtV \circ \mtG_{i} \big) \mtV^{\Tr} \mtE_{U} \vcv_{i}.
\end{aligned}
\end{equation}
\end{lemma}

With Lemma \ref{lem:app:GiMatrix} in place, use \eqref{eq:app:outputDiffGFTfirstTermFirstSolved} in \eqref{eq:app:outputDiffGFTfirstTermFirst} and \eqref{eq:app:outputDiffGFTfirstTermSecondSolved} in \eqref{eq:app:outputDiffGFTfirstTermSecond}, and this in turn back into \eqref{eq:app:outputDiffGFT} to obtain
\begin{equation} \label{eq:app:outputDiffGFTsolved}
\begin{aligned}
    \big( \mtfnH^{\text{nv}} & (\mthS)-\mtfnH^{\text{nv}}(\mtS) \big) \vcx \\ & = \sum_{i=1}^{N} \sctx_{i} m_{i} \diag(\vcfnth'(\lambda_{i}))\vcv_{i} \\ & \quad + \sum_{i=1}^{N} \sctx_{i} \big( \mtV  \circ \mtG_{i} \big) \mtV^{\Tr} \mtE_{U} \vcv_{i} \\ & \quad  + \sum_{i=1}^{N} \sctx_{i} \mtD \vcv_{i}.
\end{aligned}
\end{equation}
Next, compute the norm of \eqref{eq:app:outputDiffGFTsolved} by applying the triangle inequality to compute the norms of each of the three summands. For the norm of the first term in \eqref{eq:app:outputDiffGFTsolved}, the triangle inequality gives
\begin{equation}
\begin{aligned}
    \Big\| \sum_{i=1}^{N} & \sctx_{i} m_{i} \diag(\vcfnth'(\lambda_{i}))\vcv_{i} \Big\| \\ & \leq \sum_{i=1}^{N} | \sctx_{i} | |m_{i}| \| \diag(\vcfnth'(\lambda_{i})) \| \| \vcv_{i}\|
\end{aligned}
\end{equation}
where the submultiplicative property of the operator norm was used to bound $\| \diag(\vcfnth'(\lambda_{i})) \vcv_{i}\| \leq \| \diag(\vcfnth'(\lambda_{i})) \| \| \vcv_{i}\|$. Leveraging that $\|\vcv_{i}\|=1$, that $\|\diag(\vcfnth'(\lambda_{i}))\| = \max_{j} |\fnth'_{j}(\lambda_{i})| \leq C$ by Lipschitz continuity, and that $|m_{i}| \leq \sceps$ by the hypothesis that $\|\mtS - \mthS\| \leq \sceps$, the norm can be further bounded as
\begin{equation}\label{eq:app:lipschitzIneq1}
\begin{aligned}
    \Big\| \sum_{i=1}^{N} \sctx_{i} m_{i} & \diag(\vcfnth'(\lambda_{i})) \vcv_{i} \Big\| \\ & \leq \sceps C \sum_{i=1}^{N} | \sctx_{i} | = \sceps C \| \vctx \|_{1} \\ & \leq \sceps C \sqrt{N} \| \vctx \|= \sceps C \sqrt{N} \| \vcx \|.
\end{aligned}
\end{equation}
Note that the inequality between the $1$-norm and the $2$-norm was used, as well as the fact that the GFT is a Parseval operator. For the norm of the second term in \eqref{eq:app:outputDiffGFTsolved}, the triangle inequality together with the submultiplicativity of the operator norm are used to get
\begin{equation}\label{eq:app:lipschitzIneq2}
\begin{aligned}
    \Big\| & \sum_{i=1}^{N} \sctx_{i} \big(\mtV \circ \mtG_{i}\big) \mtV^{\Tr} \mtE_{U} \vcv_{i} \Big\| \\ & \leq \sum_{i=1}^{N} | \sctx_{i} | \| \mtV \circ \mtG_{i} \| \| \mtV^{\Tr} \| \| \mtE_{U}\| \| \vcv_{i}\| \\ & \leq 8\sceps CN\sqrt{N} \| \vcx\|.
\end{aligned}
\end{equation}
For the last inequality to hold, denote by $\|\cdot\|_{\max}$ the entrywise maximum norm of a matrix, and note that $\|\vcV \circ \mtG_{i} \| \leq N \| \vcV \circ \mtG_{i}\|_{\max} = N \max_{t,j \in \{1,\ldots,N\}} (| [\vcV]_{tj}|\ |[\mtG_{i}]_{tj}| ) \leq NC$ for all $i \in \{1,\ldots,N\}$, by the Lipschitz continuity of the frequency responses and the fact that $\mtV$ is an orthogonal matrix, so $|[\vcV]_{tj}|\le 1$ for all $t$ and $j$. Also recall that $\|\mtE_{U}\| \leq 8 \sceps$. Finally, note that the inequality between the $1$-norm and $2$-norm of vectors together with the Parseval nature of the GFT were used. For the third term in \eqref{eq:app:outputDiffGFTsolved}, it simply holds that
\begin{equation}\label{eq:app:lipschitzIneq3}
\begin{aligned}
    \Big\| & \sum_{i=1}^{N} \sctx_{i} \mtD \vcv_{i} \Big\| = \Big\| \mtD \sum_{t=1}^{N} \sctx_{i} \vcv_{i} \Big\| = \| \mtD \vcx \| \\ & \leq \| \mtD \| \| \vcx \| \leq \bigOh(\sceps^{2}) \| \vcx\|.
\end{aligned}
\end{equation}
Substituting \eqref{eq:app:lipschitzIneq1}, \eqref{eq:app:lipschitzIneq2}, and \eqref{eq:app:lipschitzIneq3} into \eqref{eq:app:outputDiffGFTsolved} gives \eqref{eq:stability}, which concludes the proof.
\end{proof}

\begin{proof}[Proof of Lemma \ref{lem:app:GiMatrix}]
We have for all $i\in\{1,\ldots,N\}$ that
\begin{equation}\label{eq:app:lemEq1}
\begin{aligned}
    & \sum_{k=0}^{K} \diag(\vch^{(k)}) \sum_{r=0}^{k-1}\lambda_{i}^{k-r-1}\mtV \diag(\vclambda^{r}) \\ & = \sum_{k=0}^{K} \diag(\vch^{(k)}) \mtV \sum_{r=0}^{k-1}\lambda_{i}^{k-r-1} \diag(\vclambda^{r}).
\end{aligned}
\end{equation}
The last summation $\sum_{r=0}^{k-1}\lambda_{i}^{k-r-1} \diag(\vclambda^{r})$ is a diagonal matrix, where the $j^{\text{th}}$ element of the diagonal can be written as
\begin{equation}
\begin{aligned}
    \Big[& \sum_{r=0}^{k-1} \lambda_{i}^{k-r-1} \vclambda^{r} \Big]_{j} = \sum_{r=0}^{k-1} \lambda_{i}^{k-r-1}\lambda_{j}^{r} \\ & = [\vcgamma_{i}]_{j} \coloneqq \begin{cases}
    k \lambda_{i}^{k-1} & \text{ if } j=i \\
    \frac{\lambda_{i}^{k} - \lambda_{j}^{k}}{\lambda_{i}-\lambda_{j}} & \text{ if } j \neq i
    \end{cases}
\end{aligned}
\end{equation}
so that
\begin{equation}\label{eq:app:lemEq2}
\begin{aligned}
    & \sum_{k=0}^{K} \diag(\vch^{(k)}) \mtV \sum_{r=0}^{k-1}\lambda_{i}^{k-r-1} \diag(\vclambda^{r}) \\ & = \sum_{k=0}^{K} \diag(\vch^{(k)}) \mtV \diag(\vcgamma_{i}).
\end{aligned}
\end{equation}
Note that the $i$ subindex in $\vcgamma_{i}$ indicates that this is parametrized by $\lambda_{i}$, while each entry of this vector, i.e., the $j^{\text{th}}$ entry, actually depends on $\lambda_{j}$. Now, remark that since we have a full matrix $\mtV$ in between two diagonal matrices, the matrix product does not commute.

To continue simplifying the expressions, start by considering a vector $\vca \in \fdR^{N}$, a matrix $\mtB \in \fdR^{N \times N}$ and another vector $\vcc \in \fdR^{N}$. Observe that
\begin{align*}
    \big[ \diag(\vca) \mtB \diag(\vcc)\big]_{ij} & = a_{i}c_{j}b_{ij} \\
    \diag(\vca) \mtB \diag(\vcc) & = \mtB \circ (\vca \vcc^{\Tr}).
\end{align*}
Therefore, it can be written
\begin{equation}
    \sum_{k=0}^{K} \diag(\vch^{(k)}) \mtV \diag(\vcgamma_{i}) = \mtV \circ \Big( \sum_{k=0}^{K} \vch^{(k)} \vcgamma_{i}^{\Tr} \Big) = \mtV \circ \mtG_{i}
\end{equation}
where $\mtG_{i} \in \fdR^{N \times N}$ is the matrix defined in \eqref{eq:app:GiMatrix};
\begin{equation}
\begin{aligned}
    \sum_{k=0}^{K} h_{tk} [\vcgamma_{i}]_{j} & = \begin{cases}
        \sum_{k=0}^{K} h_{tk}k \lambda_{i}^{k-1} & \text{ if }j=i \\
        \sum_{k=0}^{K} h_{tk}\frac{\lambda_{i}^{k} - \lambda_{j}^{k}}{\lambda_{i}-\lambda_{j}} & \text{ if } j \neq i
    \end{cases}
    \\ & =
    \begin{cases}
        \fnth_{t}'(\lambda_{i}) & \text{ if }j=i \\
        \frac{\fnth_{t}(\lambda_{i})-\fnth_{t}(\lambda_{j})}{\lambda_{i}-\lambda_{j}} & \text{ if } j \neq i
    \end{cases}
    = [\mtG_{i}]_{tj}.
\end{aligned}
\end{equation}
Using this result in \eqref{eq:app:lemEq1} and \eqref{eq:app:lemEq2} gives that
\begin{equation}
    \sum_{k=0}^{K} \diag(\vch^{(k)}) \sum_{r=0}^{k-1}\lambda_{i}^{k-r-1}\mtV \diag(\vclambda^{r}) = \mtV \circ \mtG_{i}.
\end{equation}
This equality, together with the fact that $\mtS^{r} = \mtV \diag(\vclambda^{r}) \mtV^{\Tr}$, gives the desired result \eqref{eq:app:outputDiffGFTfirstTermSecondSolved}.
\end{proof}


\section{Optimal Unbiased Node-Variant Graph Filters}

\begin{proof}[Proof of Lemma~\ref{l:unbiased}]
The NVGF-based estimator is given by:
\begin{equation} \label{eq:app:NVGFestimator}
    \vchy = \mtfnH^{\text{nv}}(\mtS) \vcx + \vcc.
\end{equation}
This estimator is unbiased if and only if $\xp[\vchy] = \xp[\vcy] = \xp[\fnrho(\vcx)] = \vcmu_{\rho}$. By \eqref{eq:app:NVGFestimator} and linearity of expectation, this holds precisely when
\begin{equation}
    \xp[\vchy] = \mtfnH^{\text{nv}}(\mtS) \xp[\vcx] + \vcc = \mtfnH^{\text{nv}}(\mtS) \vcmu_{x} + \vcc = \vcmu_{\rho}.
\end{equation}
This is equivalent to the condition on $\vcc$ that
\begin{equation}
    \vcc = \vcmu_{\rho} - \mtfnH^{\text{nv}}(\mtS)\vcmu_{x},
\end{equation}
which completes the proof.
\end{proof}

\begin{proof}[Proof of Proposition~\ref{prop:NVGFoptimal}]
Given an unbiased estimator $\vchy = \mtfnH^{\text{nv}}(\mtS) ( \vcx - \vcmu_{x}) + \vcmu_{\rho}$, an optimal matrix $\mtH^{\opt} \in \fdR^{N \times (K+1)}$ of filter taps are the ones that minimize \eqref{eq:approx}. First, note that the estimator output at node $\lmv_{i}$, i.e., the $i^{\text{th}}$ entry of $\vchy$, is given by
\begin{equation} \label{eq:app:singleNodeOutput}
    \schy_{i} = [\vchy]_{i} = \vcu_{i}^{\Tr} \diag( \mtLambda \vch_{i}) \mtV^{\Tr} (\vcx-\vcmu_{x}) + \scmu_{\rho i}
\end{equation}
where $\scmu_{\rho i} = [\vcmu_{\rho}]_{i}$, $\vcu_{i} \in \fdR^{N}$ is the $i^{\text{th}}$ row of $\mtV$, and $\vch_{i} \in \fdR^{K+1}$ is the $i^{\text{th}}$ row of the matrix $\mtH$; see \cite{Segarra2017-GraphFilterDesign}. Note that if $\vca, \vcb, \vcc$ are three $N$-dimensional real vectors, then $\vca^{\Tr} \diag(\vcb) \vcc = \sum_{i=1}^{N} a_{i}b_{i}c_{i}$, which means that they commute, i.e., $\vca^{\Tr} \diag(\vcb) \vcc = \vcc^{\Tr} \diag(\vca) \vcb$. Using this fact in \eqref{eq:app:singleNodeOutput} yields
\begin{equation} \label{eq:app:singleNodeWithA}
    \schy_{i} = (\vcx - \vcmu_{x})^{\Tr} \mtV \diag(\vcu_{i}) \mtLambda \vch_{i} + \scmu_{\rho i} = \vcfna_{i}(\vcx)^{\Tr} \vch_{i} + \scmu_{\rho i}
\end{equation}
with
\begin{equation*}
    \vcfna_{i}(\vcx) = \mtLambda^{\Tr} \diag(\vcu_{i}) \mtV^{\Tr} (\vcx - \vcmu_{x}).
\end{equation*}

The objective function in \eqref{eq:approx} can be rewritten as
\begin{equation} \label{eq:app:approx}
    \xp \big[ \| \vchy - \vcy \|_2^{2} \big] = \xp \Big[ \sum_{i=1}^{N} (\schy_{i}-y_{i})^{2} \Big] = \sum_{i=1}^{N} \xp\big[ (\schy_{i}-y_{i})^{2} \big].
\end{equation}
Since each $\schy_{i}$ depends only on $\vch_{i}$, as indicated in \eqref{eq:app:singleNodeWithA}, minimizing \eqref{eq:app:approx} over $\mtH$ is equivalent to minimizing each of the summands in \eqref{eq:app:approx} over $\vch_{i}$. Therefore, \eqref{eq:approx} becomes equivalent to the following system of $N$ optimization problems over each of the rows of $\mtH$:
\begin{equation} \label{eq:app:singleApprox}
    \min_{\vch_{i} \in \fdR^{K+1}} \xp\big[ (\schy_{i}-y_{i})^{2} \big], \quad i \in \{1,\ldots,N\}.
\end{equation}
Substituting \eqref{eq:app:singleNodeWithA} into the $i^{\text{th}}$ objective function of \eqref{eq:app:singleApprox} gives
\begin{equation} \label{eq:app:singleObjective}
\begin{aligned}
    & \xp \big[ (\schy_{i}-y_{i})^{2} \big] = \xp \big[ \big(\vcfna_{i}(\vcx)^{\Tr} \vch_{i} + \mu_{\rho i} - \rho(x_{i}) \big)^{2} \big] \\
    & = \vch_{i}^{\Tr} \xp\big[\vcfna_{i}(\vcx) \vcfna_{i}(\vcx)^{\Tr}\big]\vch_{i} \\ & \quad - 2 \xp \big[ (\rho(x_{i})-\mu_{\rho i}) \vcfna_{i}(x)^{\Tr} \big]\vch_{i} \\ & \quad + \xp \big[ (\rho(x_{i})- \mu_{\rho i})^{2} \big].
\end{aligned}
\end{equation}
Now, with
\begin{align*}
    \mtR_{i} & = \xp \big[\vcfna_{i}(\vcx) \vcfna_{i}(\vcx)^{\Tr} \big] \\ & = \mtLambda^{\Tr} \diag(\vcu_{i}) \mtV^{\Tr} \mtC_{x} \mtV \diag(\vcu_{i}) \mtLambda, \label{eq:app:Ri} \\
    \vcp_{i} & = \xp \big[ (\rho(x_{i})-\mu_{\rho i}) \vcfna_{i}(x) \big] \\ & = \mtLambda^{\Tr} \diag(\vcu_{i}) \mtV^{\Tr} \xp \big[ (\rho(x_{i})- \mu_{\rho i}) (\vcx - \vcmu_{x}) \big] 
\end{align*}
the $i^\text{th}$ objective \eqref{eq:app:singleObjective} becomes
\begin{equation} \label{eq:app:paraboloid}
    \xp \big[ (\schy_{i}-y_{i})^{2} \big] = \vch_{i}^{\Tr} \mtR_{i} \vch_{i} - 2 \vcp_{i}^{\Tr} \vch_{i} + \xp \big[ (\rho(x_{i})-\mu_{\rho i})^{2} \big].
\end{equation}
Since $\mtR_{i}$ is a positive semidefinite matrix, \eqref{eq:app:paraboloid} is a convex quadratic function in $\vch_i$. Therefore, setting the gradient of this function to zero, it can be concluded that $\vch_i^{\opt}$ is a global minimizer of \eqref{eq:app:paraboloid} if and only if $2\mtR_i \vch_i^{\opt} - 2\vcp_i = \vcZeros$, or, equivalently,
\begin{equation}
    \mtR_{i} \vch_{i}^{\opt} = \vcp_{i}.
\end{equation}
This completes the proof.
\end{proof}

%% file: AB-appendixSims.tex
\section{Authorship Attribution}

\textbf{Problem objective.} Consider a set $\stTheta_{a}$ of texts that are known to be written by author $a$. Given a new text $\theta \notin \stTheta_{a}$ the objective is to determine whether $\theta$ was written by $a$ or not.

\textbf{Approach.} The approach is to leverage word adjacency networks (WAN) built from the set of known texts $\stTheta_{a}$ to build a graph support $\stG$, and then use the word frequency count of the function words in $\theta$ as the graph signal $\vcx$. Then, $\vcx$ is processed through a GNN $\fnPhi$ (in any of the variants discussed in Section~\ref{sec:archit}) to obtain a predictor of the text being written by author $a$.

\begin{figure*}
    \centering
    \includegraphics[width=0.9\textwidth]{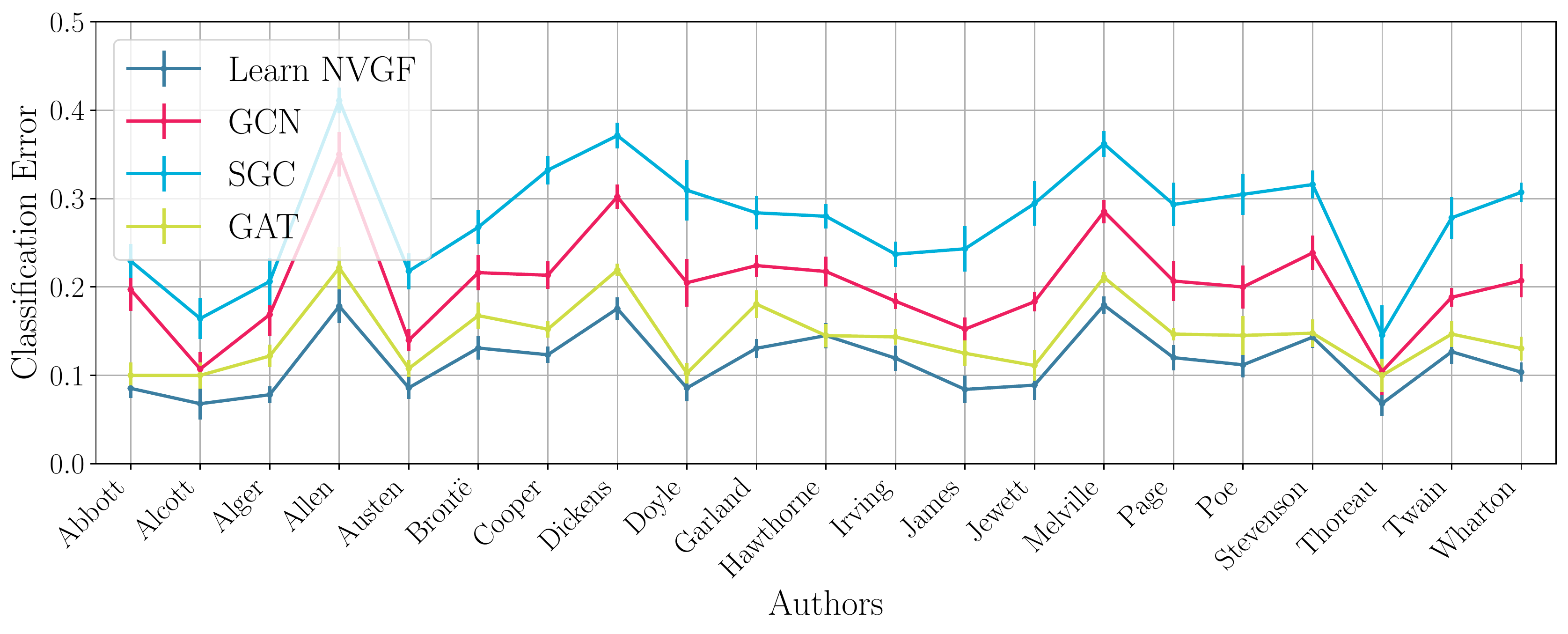}
    \caption{Comparison between architectures for all authors. It is observed that, in most cases, the Learn NVGF exhibits significantly better performance than the GAT \cite{Velickovic2018-GAT}, which in turn is better than the GCN \cite{Kipf2017-GCN}, and all of them are better than the SGC \cite{Weinberger2019-SGC}. The error bars reflect one third of the estimated standard deviation.}
    \label{fig:app:allAuthors:compare}
\end{figure*}

\textbf{Graph construction.} The WAN can be modeled by a graph $\stG$ where the set of nodes $\stV$ consists of function words (i.e., words that do not carry semantic meaning but express grammatical relationships among other words within a sentence, e.g., ``the'', ``and'', ``a'', ``of'', ``to'', ``for'', ``but''). The existence of an edge connecting words and the corresponding edge weight are determined as follows. Consider each text $\theta \in \stTheta_{a}$ and split it into a total of $S$ sentences $\{\fnomega_{\theta}^{s}\}_{s = 1}^S$ where each sentence $\fnomega_{\theta}^{s}: \fdN \to \stV \cup \{\emptyset\}$ gives the function word present in each position within a sentence, or $\emptyset$ if the word is not a function word. Then, given a discount factor $\alpha \in (0,1)$ and a window length $D$, the edge weight $w_{ij}$ between nodes $v_{i}$ and $v_{j}$ is computed as %
\begin{equation} \label{eq:app:authorWeights}
\begin{aligned}
    w_{ij} = \sum_{t : \theta_{t} \in \stTheta_{a}} & \sum_{s,e} \indFn \{\fnomega_{\theta_{t}}^{s} (e) = v_{i}\} \\ & \sum_{d=1}^{D} \alpha^{d-1} \indFn\{\fnomega_{\theta_{t}}^{s}(e+d) = v_{j}\}
\end{aligned}
\end{equation}
where $\indFn\{\stA\}$ is the indicator function that takes value $1$ when condition $\stA$ is met and 0 otherwise, see \cite{Segarra2015-Authorship}. Equation \eqref{eq:app:authorWeights} essentially computes each weight by going text by text $\theta_{t} \in \stTheta_{a}$ and sentence by sentence $s \in \{1,\ldots,S\}$, looking position by position $e$ for the corresponding word $\fnomega_{\theta_{t}}^{s}(e)$ to match the function word $v_{i}$. Once the word $v_{i}$ is matched, the following $D$ words in the window length are looked at and, if the $(e+d)^{\text{th}}$ word matches $v_{j}$, then the discounted weight $a^{d-1}$ is added. In this way, not only co-occurrence of words, but also their proximity counts in establishing the WAN. Note that the edge weight function \eqref{eq:app:authorWeights} is asymmetric, which results in a directed graph.

\textbf{Graph signal processing description.} The support matrix is chosen to be
\begin{equation} \label{eq:app:authorSupport}
    \mtS = \frac{1}{2}(\mtD^{-1} \mtW + \mtW^{\Tr} \mtD^{-1})
\end{equation}
where the $\mtW$ is the adjacency matrix with entry $(i,j)$ equal to $w_{ij}$ and $\mtD = \diag(\mtW \vcOnes)$ is the degree matrix. The operation $\mtD^{-1}\mtW$ normalizes the matrix by rows, and the support matrix $\mtS$ comes from symmetrizing the matrix by adopting one half of the weight on each direction. The graph signal $\vcx$ contains a normalized word frequency count for each word
\begin{equation} \label{eq:app:authorSignal}
    [\vcx]_{i} = \frac{\sum_{s,e} \indFn\{\fnomega_{\theta}^{s}(e) = v_{i}\}}{\sum_{j : v_{j} \in \stV \sum_{s,e}} \indFn\{\fnomega_{\theta}^{s}(e) = v_{j}\}}.
\end{equation}
Note that the graph signals are normalized by the $1$-norm and can therefore be interpreted as the probability of finding the function word $v_{i}$ in text $\theta$.

\textbf{Dataset and code.} The dataset is presented in \cite{Segarra2015-Authorship}, and is publicly available at \url{http://github.com/alelab-upenn/graph-neural-networks/tree/master/datasets/authorshipData}. The dataset consists of $21$ authors from the $\text{19}^{\text{th}}$ century. The corpus of each author is split in texts of about $1000$ words. The WANs and the word frequency count for each text are already present in the dataset. The texts are split, at random, in $95\%$ for training and $5\%$ for testing. The weights of the WANs of the $95\%$ texts selected for training are averaged to obtain an average WAN from which the support matrix $\mtS$ is obtained by following \eqref{eq:app:authorSupport}. The resulting support matrix is further normalized to have unit spectral norm. Note that no text from the test set is used in building the WAN. From the texts in the training set, $8\%$ are further separated to build the validation set. Denote by $N_{a}^{\text{train}}$, $N_{a}^{\text{valid}}$, and $N_{a}^{\text{test}}$ the number of texts in the training, validation, and test set, respectively.  The word frequency counts for each text are normalized as in \eqref{eq:app:authorSignal} and used as graph signals. A label of $1$ is attached to these signals to indicate that they have been written by author $a$ in a supervised learning context. To complete the datasets, an equal number of texts are obtained at random from other contemporary authors, and their frequency word counts are normalized and incorporated into the corresponding sets and assigned a label of $0$. In this way, the resulting training, validation, and test set have $2N_{a}^{\text{train}}$, $2N_{a}^{\text{valid}}$, and $2N_{a}^{\text{test}}$ samples, respectively (half of them labeled with $1$ and the other half with $0$). The code to run the simulations will be provided as a .zip file.

\begin{figure*}
    \centering
    \includegraphics[width=0.9\textwidth]{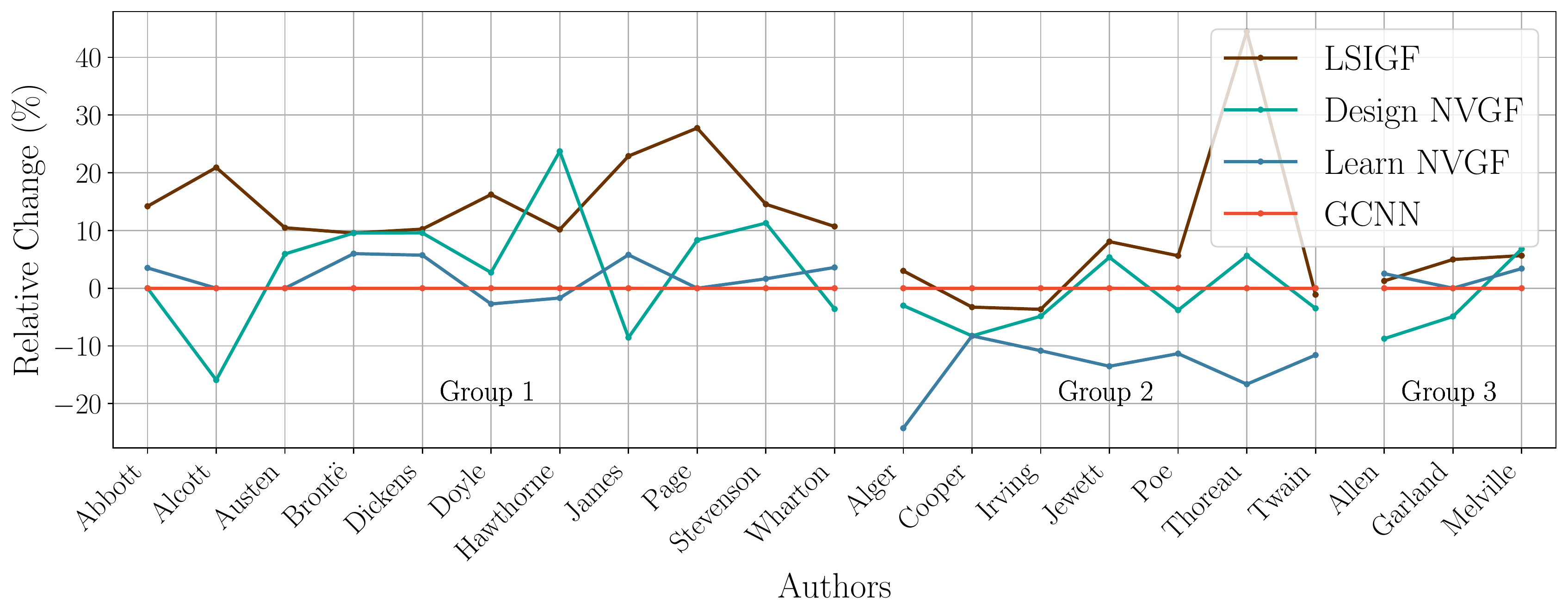}
    \caption{Relative change in performance with respect to the GCNN baseline, divided into $3$ groups of similar behavior. Group 1 ($53\%$ of the authors) includes those where the Learn NVGF has a comparable performance to the GCNN and both of them are better than the LSIGF, showing that frequency creation plays a vital role in improving performance. Group 2 ($33\%$ of the authors) shows that oftentimes, the Learn NVGF can improve significantly over the GCNN, suggesting that the nonlinear nature of the mapping may have a negative impact. Group 3 ($14\%$ of the authors) consists of those cases when the Learn NVGF, the GCNN, and the LSIGF all exhibit comparable performance.}
    \label{fig:app:allAuthors:change}
\end{figure*}

\textbf{Architectures for comparison.} The Learn NVGF architecture in \eqref{eq:learnNVGF} is compared against three of the most popular GNN architectures in the literature, namely, the GCN \cite{Kipf2017-GCN}, the SGC \cite{Weinberger2019-SGC}, and the GAT \cite{Velickovic2018-GAT}. The Learn NVGF adopts the support matrix in \eqref{eq:app:authorSupport} and consists of a LSI graph filter (a graph convolution) that outputs $F_{\text{NVGF}}$ features (hidden units) and has $K_{\text{NVGF}}+1$ filter taps, followed by a NVGF that takes those $F_{\text{NVGF}}$ input features and applies a NVGF with $K_{\text{NVGF}}+1$ filter taps, and also outputs $F_{\text{NVGF}}$ features. The GCN considers $\mtS$ in \eqref{eq:app:authorSupport} to be the adjacency matrix of the graph and thus adopts a support matrix given by $\mtS_{\text{GCN}} = \mttD^{-1/2}(\mtI + \mtS) \mttD^{-1/2}$ where $\mttD = \diag((\mtI+\mtS) \vcOnes)$ as indicated in \cite{Kipf2017-GCN}. The graph convolutional layer consists of $F_{\text{GCN}}$ LSI graph filters, each one of the form $h^{f}\mtS_{\text{GCN}}$ for $f\in\{1,\ldots,F_{\text{GCN}}\}$. This is followed by a ReLU activation function. The SGC considers the same support matrix $\mtS_{\text{GCN}}$ as the GCN and learns $F_{\text{SGC}}$ output features, where each filter in the bank is of the form $h^{f} \mtS_{\text{GCN}}^{K_{\text{SGC}}}$ for a predetermined hyperparameter $K_{\text{SGC}}$. This is followed by a ReLU activation function. Finally, for the GAT, the support matrix is $\mtS$ (although this is relevant only in terms of the nonzero elements, which are the same in $\mtS$ and $\mtS_{\text{GCN}}$---except for the diagonal elements---since the weights of each edge are learned through the attention mechanism), and the output features are $F_{\text{GAT}}$ learned through the attention mechanism exactly as described in \cite{Velickovic2018-GAT}. All four architectures are followed by a readout layer consisting of a learnable linear transform that maps the $NF$ output features into a vector of size $2$ that is interpreted to be the logits for the two classes (either the text is written by the author or not).

\textbf{Architectures for analysis.} To analyze the impact of the nonlinearity, four architectures are considered. First, as a baseline, a GCNN consisting of a graph convolutional layer that outputs $F$ features, with $K+1$ filter taps, followed by a ReLU activation function, as indicated in \eqref{eq:GCNN}. Second, the Learn NVGF that replaces the ReLU activation function by a NVGF with filter taps that can be learned from data as in \eqref{eq:learnNVGF}. Third, the Design NVGF where the nonlinear activation function of the GCNN is replaced by a NVGF, but one whose filter taps are designed to mimic the ReLU (the LSI graph filters in the Design NVGF are the same ones learned by the GCNN). Fourth, a LSI graph filter with $K+1$ taps that outputs $F$ features. Note that the GCNN is a nonlinear architecture, while the other three are linear. Also, note that the first three architectures are capable of creating frequency content while the fourth one, the LSIGF, is not. For fair comparison, the values of $K$ and $F$ are the same for all architectures.

\textbf{Training.} The loss function during training is a cross-entropy loss between the logits obtained from the output of each architecture, and the labels in the training set. All the architectures are trained by using an ADAM optimizer \cite{Kingma15-ADAM} with forgetting factors $0.9$ and $0.999$, and with a learning rate $\eta$. The training is carried out for $25$ epochs with batches of size $20$. Dropout with probability $0.5$ is included before the readout layer, during training, to avoid overfitting. At testing time, the weights are correspondingly rescaled. Validation is run every $5$ training steps. The learned filters that result in the best performance on the validation set are kept and used during the testing phase. For each experiment, $10$ realizations of the random train/test split are carried out (also randomizing the selection of the texts by other authors that complete the training, validation, and test sets). The average evaluation performances (measured as classification error---ratio of texts wrongly attributed in the test set) is reported, together with the estimated standard deviation.

\textbf{Hyperparameter selection.} The number of hidden units $F_{\text{NVGF}}$, $F_{\text{GCN}}$, $F_{\text{SGC}}$, and $F_{\text{GAT}}$, the polynomial order $K_{\text{NVGF}}$ and $K_{\text{SGC}}$, and the learning rate $\eta_{\text{NVGF}}$, $\eta_{\text{GCN}}$, $\eta_{\text{SGC}}$, and $\eta_{\text{GAT}}$ are selected, independently for each architecture, from the set $\{16,32,64\}$ for the number of hidden units, $\{2, 3, 4\}$ for the polynomial order, and $\{0.001, 0.005, 0.01\}$ for the learning rate. In other words, all possible combinations of these three parameters are run for each architecture, and the ones that show the best performance on the test set are kept.

\begin{table*}
    \caption{Hyperparameters for each architecture that lead to the best performance}
    \label{tab:app:allAuthors:hParams}
    \centering
    \small
    \begin{tabular}{l|ccc|cc|ccc|cc}
        \toprule
        & \multicolumn{3}{c|}{Learn NVGF} & \multicolumn{2}{c|}{GCN \cite{Kipf2017-GCN}} & \multicolumn{3}{c|}{SGC \cite{Weinberger2019-SGC}} & \multicolumn{2}{c}{GAT \cite{Velickovic2018-GAT}} \\
        Authors & $\eta_{\text{NVGF}}$ & $F_{\text{NVGF}}$ & $K_{\text{NVGF}}$ & $\eta_{\text{GCN}}$ & $F_{\text{GCN}}$ & $\eta_{\text{SGC}}$ & $F_{\text{SGC}}$ & $K_{\text{SGC}}$ & $\eta_{\text{GAT}}$ & $F_{\text{GAT}}$ \\
        \midrule
        Abbott & 0.001 & 32 & 2 & 0.005 & 64 & 0.01 & 64 & 2 & 0.01 & 64 \\
        Alcott & 0.01 & 16 & 4 & 0.01 & 32 & 0.01 & 32 & 2 & 0.01 & 32 \\
        Alger  & 0.005 & 16 & 4 & 0.01 & 32 & 0.001 & 64 & 2 & 0.005 & 64 \\
        Allen & 0.005 & 64 & 2 & 0.01 & 64 & 0.005 & 32 & 2 & 0.005 & 64 \\
        Austen & 0.001 & 32 & 3 & 0.01 & 64 & 0.005 & 64 & 2 & 0.01 & 64\\
        Bront\"{e} & 0.001 & 32 & 3 & 0.01 & 64 & 0.005 & 16 & 2 & 0.01 & 16 \\
        Cooper & 0.005 & 16 & 3 & 0.005 & 64 & 0.005 & 64 & 2 & 0.01 & 64 \\
        Dickens & 0.001 & 32 & 4 & 0.005 & 32 & 0.005 & 64 & 2 & 0.01 & 64 \\
        Doyle & 0.005 & 16 & 3 & 0.005 & 16 & 0.005 & 64 & 2 & 0.005 & 32 \\
        Garland & 0.005 & 32 & 3 & 0.01 & 64 & 0.005 & 64 & 2 & 0.01 & 32 \\
        Hawthorne & 0.01 & 16 & 3 & 0.005 & 64 & 0.01 & 64 & 2 & 0.01 & 16 \\
        Irving & 0.005 & 16 & 3 & 0.005 & 64 & 0.01 & 64 & 2 & 0.005 & 64 \\
        James & 0.001 & 32 & 2 & 0.005 & 64 & 0.01 & 64 & 2 & 0.01 & 64 \\
        Jewett & 0.001 & 16 & 2 & 0.01 & 64 & 0.01 & 64 & 2 & 0.01 & 32 \\
        Melville & 0.005 & 16 & 4 & 0.01 & 64 & 0.005 & 64 & 2 & 0.005 & 64  \\
        Page & 0.005 & 16 & 2 & 0.005 & 64 & 0.005 & 64 & 2 & 0.005 & 64 \\
        Poe & 0005 & 32 & 4 & 0.005 & 64 & 0.005 & 32 & 2 & 0.01 & 32 \\
        Stevenson & 0.001 & 16 & 4 & 0.005 & 64 & 0.005 & 32 & 2 & 0.01 & 64 \\
        Thoreau & 0.005 & 64 & 3 & 0.01 & 64 & 0.01 & 16 & 2 & 0.01 & 32 \\
        Twain & 0.005 & 16 & 2 & 0.01 & 16 & 0.001 & 64 & 2 & 0.01 & 32 \\
        Wharton & 0.005 & 16 & 4 & 0.01 & 64 & 0.01 & 64 & 2 & 0.001 & 64 \\
        \bottomrule
    \end{tabular}
\end{table*}

\textbf{Experiment 1: Performance comparison.} In the first experiment, the performance is measured by classification error (ratio of texts in the test set that were wrongly attributed), and the comparison between the Learn NVGF, the GCN \cite{Kipf2017-GCN}, the SGC \cite{Weinberger2019-SGC}, and the GAT \cite{Velickovic2018-GAT} is carried out, for all $21$ authors. The hyperparameters used for each architecture and each author are present in Table~\ref{tab:app:allAuthors:hParams} (recall that the hyperparameters that exhibit the best performance---the lowest classification error on the test set---are the ones used).

Results are shown in Figure~\ref{fig:app:allAuthors:compare}. The general trend observed is for the Learn NVGF to exhibit better performance than the GAT, which in turn is better than the GCN, and all of them are better than SGC. The differences are usually significant between the four architectures, with a marked improvement by the Learn NVGF. It is observed, however, that for Doyle, Hawthorne, and Stevenson, the performance of the Learn NVGF and the GAT is comparable.

\textbf{Experiment 2: Impact of nonlinearities.} In the second experiment, the objective is to decouple the contribution made by frequency creation from that made by the nonlinear nature of the architecture. To do this, the GCNN architecture is taken as a baseline (a nonlinear, frequency-creating architecture), and the relative change in performance of the three other architectures is measured (Learn NVGF and Design NVGF, both linear and frequency-creating architectures, and LSIGF, which is linear but cannot create frequencies). The learning rate $\eta$, the number of hidden units $F$, and order of the filters $K$ are the same for all four architectures, and are those found in Table~\ref{tab:app:allAuthors:hParams} under the column of Learn NVGF. The results showing relative change in the mean classification error with respect to the GCNN are shown in Figure~\ref{fig:app:allAuthors:change}.

The authors have been classified into three groups according to their relative behavior. Group 1 consists of those authors where the Learn NVGF has a comparable performance with respect to the GCNN (i.e., $5\%$ or less relative variation in performance), and both have a considerably better performance than the LSIGF. These results suggest that, for this group of authors (the most numerous one, consisting of $11$ authors, or $53\%$ of the total), the improvement in performance is mostly due to the capability of the architectures to create frequency, and not necessarily due to the nonlinear nature of the GCNN.

Group 2 consists of those authors where the Learn NVGF exhibits a better performance than the GCNN, which in turn exhibits a similar performance to the LSIGF (except for Jewett and Thoreau, where the GCNN still exhibits considerably better performance than the LSIGF). This group of 7 authors ($33\%$ of the total) suggests that in some cases, frequency creation is the key contributor to improved performance, and that the inclusion of a nonlinear mapping could possibly have a negative impact. In essence, we observe that a linear architecture capable of creating frequencies outperforms the rest, and that a frequency-creating nonlinear architecture performs similarly to a linear architecture that cannot create frequencies. This suggests that the relationship between input and output is approximately linear with frequencies being created, but that attempting to model this frequency creation with a nonlinear architecture is not a good approach.

Finally, Group 3 consists of those architectures for which all architectures exhibit a similar performance. This group consists of $3$ authors ($14\%$ of the total). This may be explained by the fact that the high-frequency content for these authors does not carry useful information, and thus the role of frequency creation is less relevant.

With respect to the Design NVGF, a somewhat erratic behavior is observed. In near half of the cases ($11$ authors), the performance of the Design NVGF closely resembles the performance of the GCNN, as expected. In a few other cases (Alcott, James, Cooper, and Allen), the Design NVGF results are better, and in the rest they are considerably worse (Bront\"{e}, Dickens, Hawthorne, Page, Stevenson, and Melville). The Design NVGF architecture is designed to mimic the GCNN, but its accurate design depends on good estimates of the first and second moments of the data. Thus, one possible explanation is that there is not enough data to get good estimates of these values. Another possible explanation is that higher-order moments have a larger impact in these cases, and the NVGF, being linear, is not able to accurately capture them.

Overall, this second experiment shows the importance of frequency-creation in improving performance, especially when high-frequency content is significant. Among the two ways of creating frequency (linear and nonlinear), we see that in most cases, they essentially perform the same. But there are cases in which creating frequency content in a linear manner is better (Group 2). In any case, this experiment shows the importance of frequency creation and calls for further research on what other contributions the nonlinear nature of the activation function has on performance.


\section{Movie Recommendation}

\begin{figure}
    \centering
    \includegraphics[width=0.9\columnwidth]{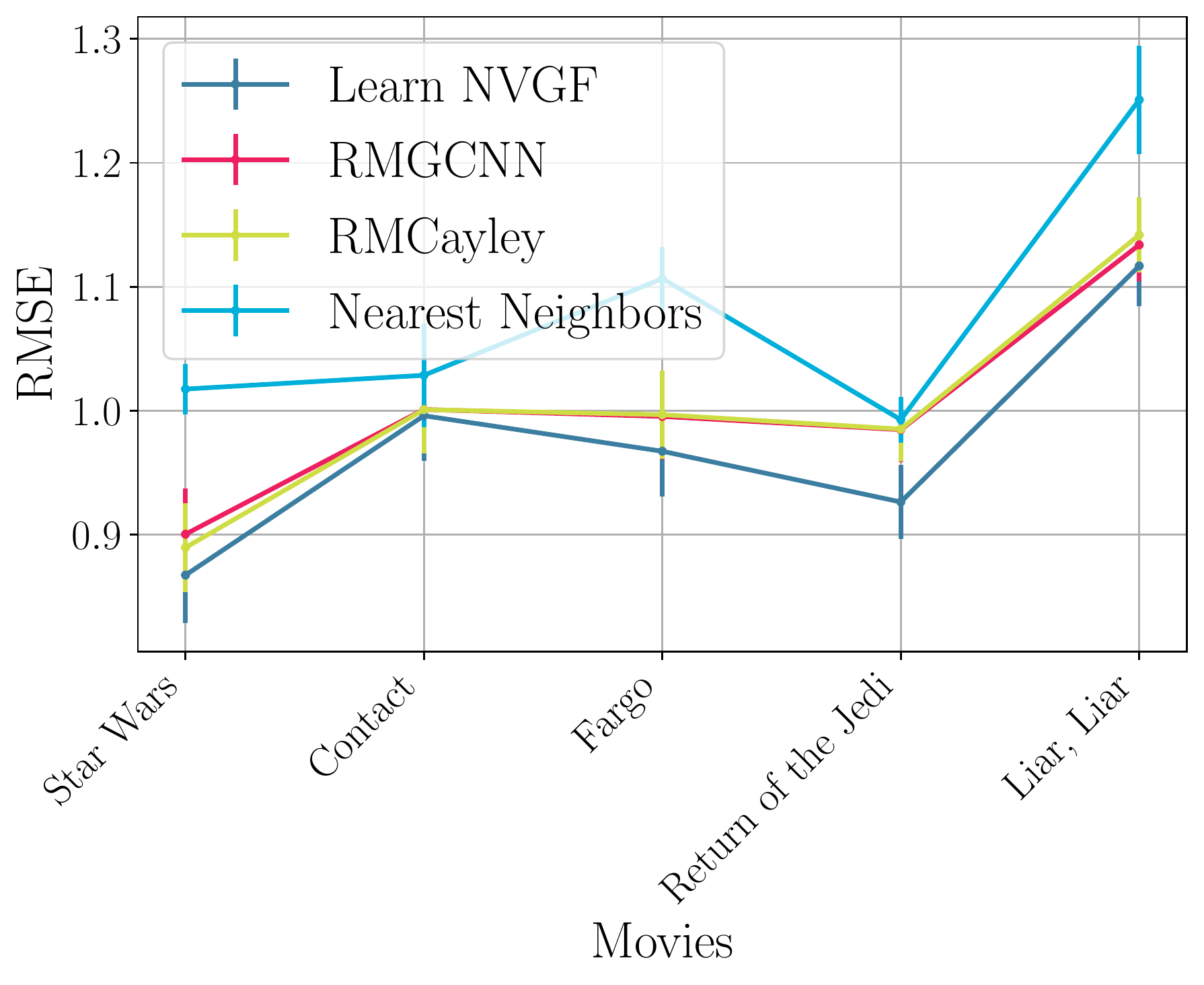}
    \caption{Performance comparison}
    \label{fig:app:allMovies:compare}
\end{figure}
\begin{figure}
    \centering
    \includegraphics[width=0.9\columnwidth]{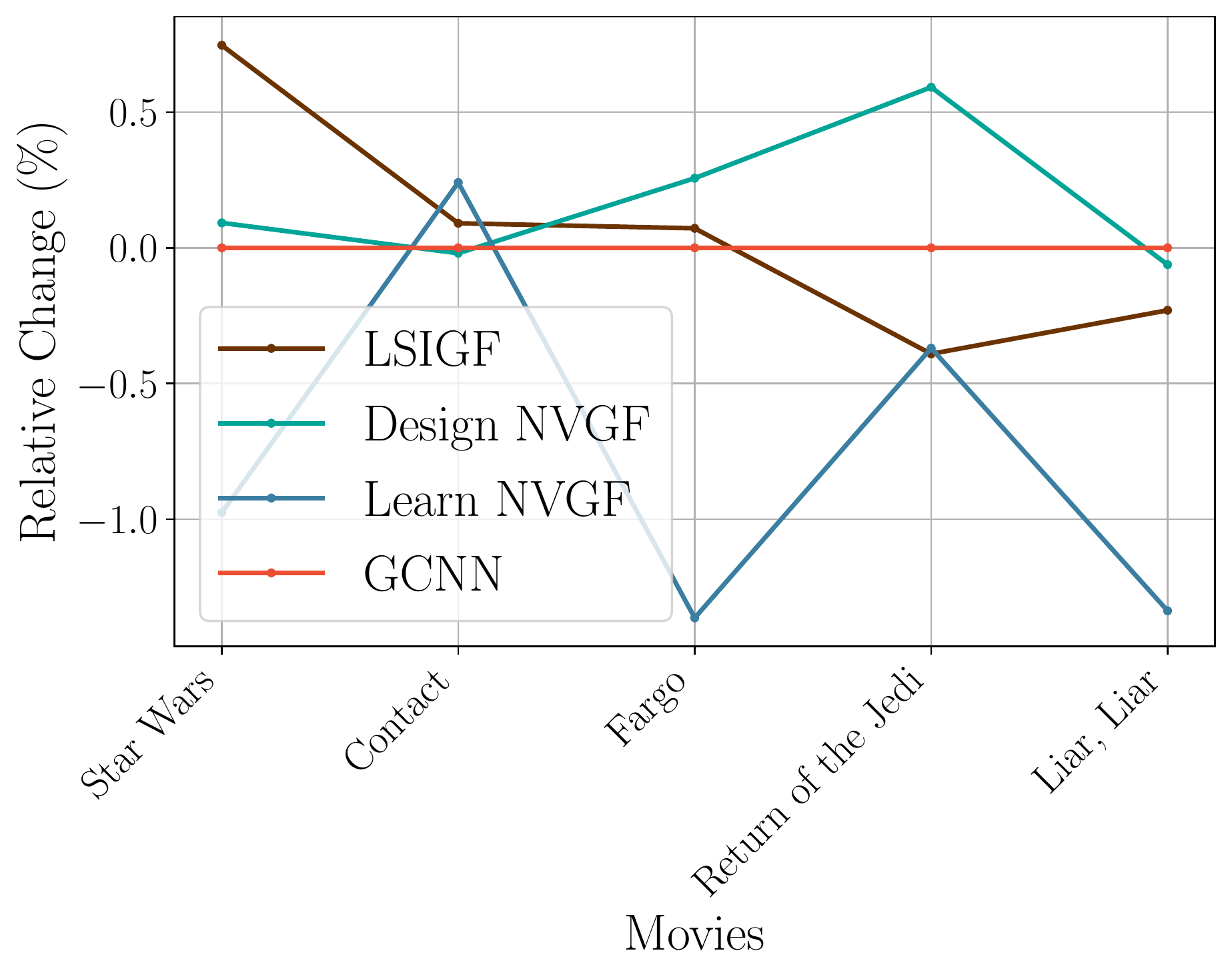}
    \caption{Relative change}
    \label{fig:app:allMovies:change}
\end{figure}

\textbf{Problem objective.} Consider a set $\stV = \{v_{1},\ldots,v_{N}\}$ of $N$ items, and let $\fntx_{t}: \stV \to \fdR \cup \{ \emptyset\}$ be the ratings assigned by user $t$ to these items, i.e., $\fntx_{t}(v_{i})$ is the rating assigned by user $t$ to item $i$, this function yields $\emptyset$ if the item has not been rated. The objective is to estimate what rating $\fntx_{t}(\lmv_{r})$ a user would give to some target item $\lmv_{r} \in \stV$ not yet rated \cite{Monti2017-RecommendationGNN, Levie2018-CayleyNets}.

\textbf{Approach.} The idea is to create a graph $\stG$ of rating similarities and take the graph signal $\vcx$ to be the ratings given by a user to some of the items. Then $\vcx$ is processed through a GNN $\fnPhi$ (in any of the variants discussed in Section~\ref{sec:archit}) to obtain the estimated rating the user would give to the target item $\lmv_{r}$, by looking at the output value of the GNN on node $\lmv_{r}$. In short, this amounts to an interpolation problem.

\textbf{Graph construction.} Let $\sttT = \{\vctx_{t}\}_{t}$ be a set where $\vctx_{t} \in \fdR^{N}$ collects the ratings given by user $t$ to some of the items, such that $[\vctx_{t}]_{i} = \fntx_{t} (\lmv_{i}) \in \fdR$ if item $\lmv_{i}$ has been rated and $[\vctx_{t}]_{i} = 0$ if $\fntx_{t}(\lmv_{i}) = \emptyset$. Denote by $\stT_{i} = \{\vcx_{t} \in \sttT : [\vcx_{t}]_{i} > 0\}$ the set of users that have rated item $\lmv_{t}$, and by $\sttT_{ij} = \sttT_{i} \cap \sttT_{j}$ the set of users that have rated both items $\lmv_{i}$ and $\lmv_{j}$. Define the mean intersection score as $\mu_{ij} = |\sttT_{ij}|^{-1} \sum_{\vctx_{t} \in \sttT_{ij}} [\vctx_{t}]_{i}$. Note that this is the rating average for item $\lmv_{i}$, computed among those users that have rated both $\lmv_{i}$ and $\lmv_{j}$. Then, the rating similarity between items $\lmv_{i}$ and $\lmv_{j}$ is computed by means of the Pearson correlation as
\begin{equation}
      w_{ij}  = \frac{1}{|\sttT_{ij}|} \sum_{\vctx_{t} \in \sttT_{ij}} \big( [\vctx_{t}]_{i} - \mu_{ij} \big) \big( [\vctx_{t}]_{j} - \mu_{ij} \big).
\end{equation}
These weights can be used to build a complete graph $\sttG = (\stV, \sttE)$ where $\stV$ is the set of items and $\sttE = \stV \times \stV$ is the complete set of edges with $w_{ij}$ being the corresponding weights. In what follows, the $10$-nearest neighbor graph $\stG = (\stV, \stE)$ of $\sttG$ is built with $\stE \subseteq \sttE$, and the matrix $\mtW$ is used to denote the weighted adjacency matrix of $\stG$, such that $[\mtW]_{ij} = w_{ij}$ if $(i,j) \in \stE$, and $[\mtW]_{ij}=0$ otherwise.

\textbf{Graph signal processing description.} The support matrix is chosen to be
\begin{equation} \label{eq:app:movieSupport}
    \mtS = \big( \diag(\mtW) \big)^{-1/2} \mtW \big( \diag(\mtW) \big)^{-1/2} - \mtI.
\end{equation}
This problem can be cast as a supervised interpolation problem. Given the set $\sttT$ and the target item $\lmv_{r}$, consider the set $\sttT_{r}$ of all users that have rated the item $\lmv_{r}$. Extract the specific rating $\fntx_{t}(\lmv_{r}) = y_{t}$ as a label, and set a $0$ in the $r^{\text{th}}$ entry of $\vctx_{t}$. The resulting vector is a graph signal $\vcx_{t}$ which always has a $0$ in the $r^{\text{th}}$ entry. The resulting set $\stT_{r} = \{(\vcx_{t},y_{t}) : \vctx_{t} \in \sttT_{r}\}$ contains all the users that have rated the item $\lmv_{r}$ with the corresopnding rating extracted as a label $y_{t}$ and the $r^{\text{th}}$ entry $[\vcx_{t}]_{r}$ of the graph signal $\vcx_{t}$ set to zero, i.e., $[\vcx_{t}]_{r} = 0$ for all $t$ such that $\vctx_{t} \in \sttT_{r}$.

\begin{figure}
    \centering
    \includegraphics[width=0.9\columnwidth]{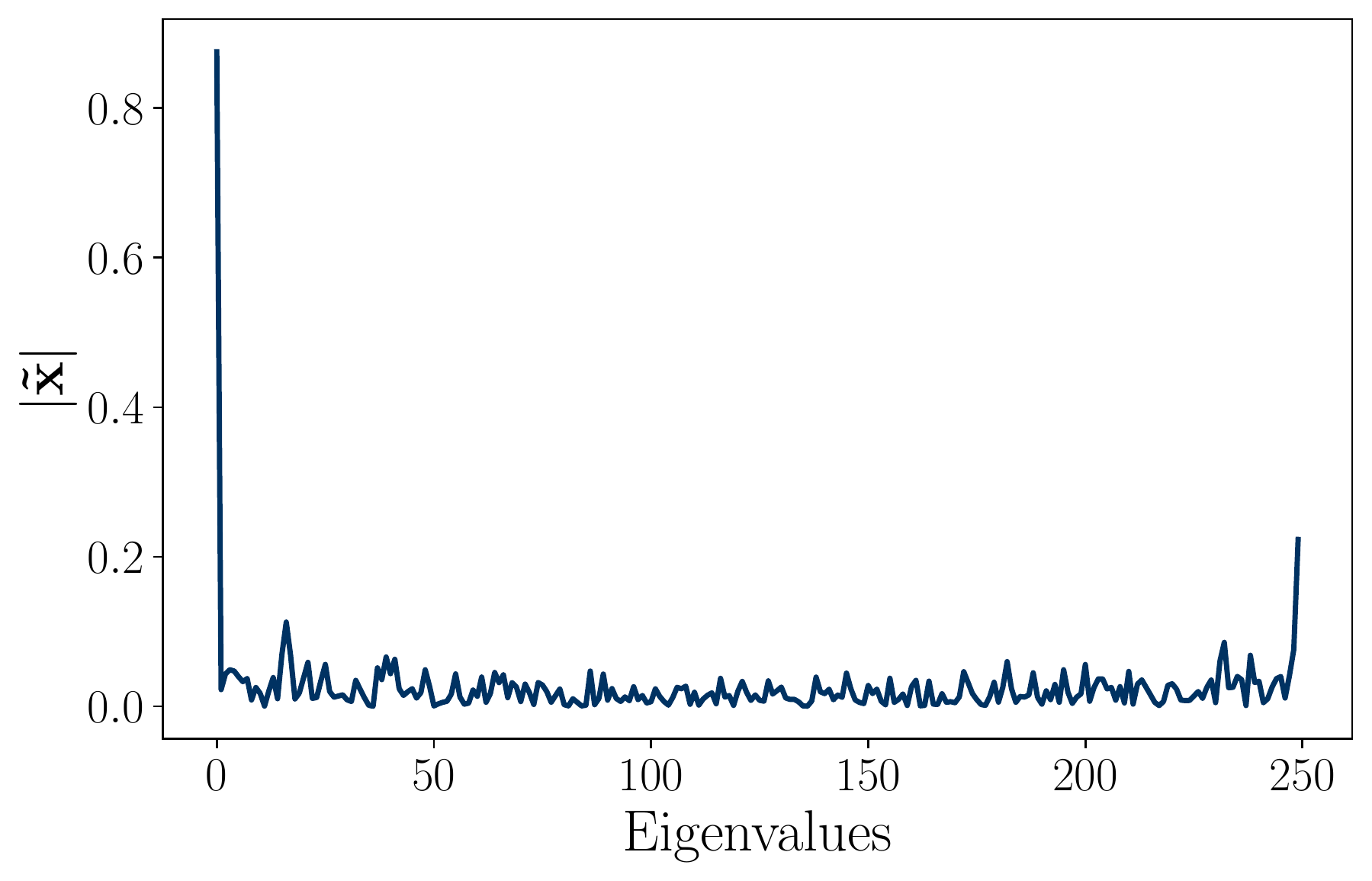}
    \caption{Input frequencies}
    \label{fig:app:starWars:inputFreq}
\end{figure}
\begin{figure}
    \centering
    \includegraphics[width=0.9\columnwidth]{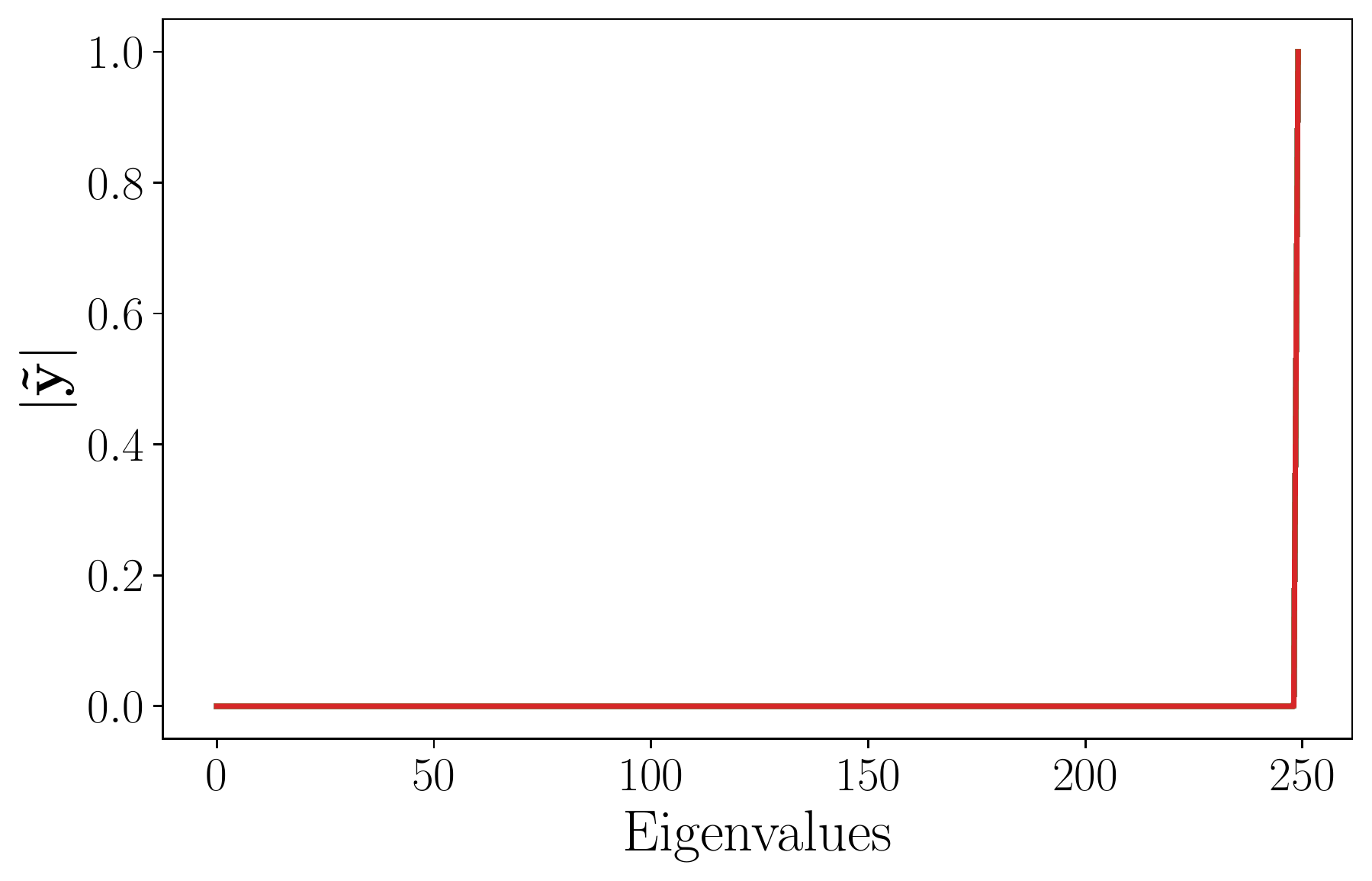}
    \caption{Output frequency of LSIGF}
    \label{fig:app:starWars:outputFreqLSIGF}
\end{figure}

\textbf{Dataset and code.} The dataset is the MovieLens-100k dataset \cite{Harper2016-MovieLens} publicly available at \url{http://files.grouplens.org/datasets/movielens/ml-100k.zip}. This dataset consists of one hundred thousand ratings given by $943$ users to $1682$ movies, and where each user has rated at least $20$ of them. The ratings are integers ranging from $1$ to $5$ meaning that $\fntx_{t} : \stV \to \{1,\ldots,5\} \cup \emptyset$ for every user $t$. In particular, the subset of $250$ movies that have received the largest number of ratings is used to build a graph with $N=250$ nodes. The resulting dataset has $54746$ ratings given by $943$ users to some of these $250$ movies. The five movies with the largest number of ratings are considered as target movies, namely ``Star Wars'' with $|\sttT_{\text{Star Wars}}| = 583$ pairs $(\vcx_{t},y_{t})$, ``Contact'' with $509$ pairs, ``Fargo'' with $508$, ``Return of the Jedi'' with $507$, and ``Liar, Liar'' with $485$. Each of these datasets is split randomly into $81\%$ for training, $9\%$ for validation, and $10\%$ for testing. The code to run the simulations will be provided as a .zip file.

\textbf{Architectures for comparison.} The Learn NVGF architecture in \eqref{eq:learnNVGF} is compared against two graph signal processing based methods for movie recommendation, namely the RMGCNN in \cite{Monti2017-RecommendationGNN} and the RMCayley in \cite{Levie2018-CayleyNets}, which are implemented exactly as in the corresponding papers, with the same values for the hyperparameters. For the Learn NVGF, the values are $F$ hidden units and filters of order $K$. The values used are $(64,3)$ for Star Wars and Contact, $(16,4)$ for Fargo, $(32,4)$ for Return of the Jedi, and $(16,3)$ for Liar, Liar. The support matrix is given by \eqref{eq:app:movieSupport}. All architectures are followed by a learnable local linear transformation (the same for all nodes, i.e., a LSI graph filter with $K=0$) that takes the value of the $F$ hidden units and outputs a single scalar that represents the estimated rating for that movie. A comparison with the nearest neighbor method (i.e., averaging the ratings of the nearest nodes) is also included.

\textbf{Architectures for analysis.} The architectures for analyzing the role of frequency creation are the same four architectures that in the authorship attribution problem. Namely, the LSI graph filter as a linear architecture unable to create frequencies, the Design NVGF and the Learn NVGF are linear frequency-creating architectures, and the GCNN with a ReLU nonlinearity is a nonlinear frequency-creating architecture. The values of $F$ and $K$ in all cases are the same.

\textbf{Training.} The loss function during training is the ``Smooth L1'' loss between the output scalar at the target node and the labels in the training set. All architectures are trained by using an ADAM optimizer \cite{Kingma15-ADAM} with forgetting factors $0.9$ and $0.999$, and a learning rate $\eta$. The training is carried out for $40$ epochs with batches of size $5$. Validation is run every $5$ training steps. The learned filters that result in the best performance on the validation set are kept and used during the testing phase. For each experiment, $10$ realizations of the random dataset split are carried out. The average evaluation performances (measured by RMSE) is reported, together with the estimated standard deviation.

\textbf{Experiment 1: Performance comparison.} In the first experiment, the performance is measured by the RMSE, and the comparison between the Learn NVGF, the RMGCNN \cite{Monti2017-RecommendationGNN}, the RMCayley \cite{Levie2018-CayleyNets}, and the Nearest Neighbor approach is carried out for the $5$ aforementioned movies with the most number of ratings. Results are shown in Figure~\ref{fig:app:allMovies:compare}. It is generally observed that the Learn NVGF performs better than the alternatives, although the performance is comparable to the RMGCNN and the RMCayley in the case of the movie Contact. The nearest neighbor approach yields worse performance.

\textbf{Experiment 2: Impact of nonlinearities.} In this second experiment, the objective is to decouple the contribution made by frequency creation from that made by the nonlinear nature of the architecture. To do this, the GCNN architecture is taken as a baseline (a nonlinear, frequency-creating architecture), and the relative change in performance of the three other architectures is measured. The results shown in Figure~\ref{fig:app:allMovies:change} show that the relative change is quite small (approximately $1.5\%$ change in the highest case, the movies Fargo and Liar, Liar), which implies that all four architectures have relatively similar performance. This can be easily explained by computing the average frequency response of the signals in the test set of the movie Star Wars. The result is shown in Figure~\ref{fig:app:starWars:inputFreq}. It is observed that it is a signal with low-eigenvalue frequency content. Therefore, as expected, there is not much to gain for using architectures that create frequencies. In Figures~\ref{fig:app:starWars:outputFreqLSIGF}, \ref{fig:app:starWars:outputFreqGCNN}, and \ref{fig:app:starWars:outputFreqLearn}, the frequency responses of the output of each architecture to an input that is equal to the largest eigenvector, i.e., $\vcx = \vcv_{N}$, are shown. As expected, the LSIGF (Figure~\ref{fig:app:starWars:outputFreqLSIGF}) does not create frequency content, while the other two architectures, do (Figures~\ref{fig:app:starWars:outputFreqGCNN} and \ref{fig:app:starWars:outputFreqLearn}). However, since this high-eigenvalue frequency content is not really significant, the frequency creation capabilities do not markedly improve the performance.

\begin{figure}
    \centering
    \includegraphics[width=0.9\columnwidth]{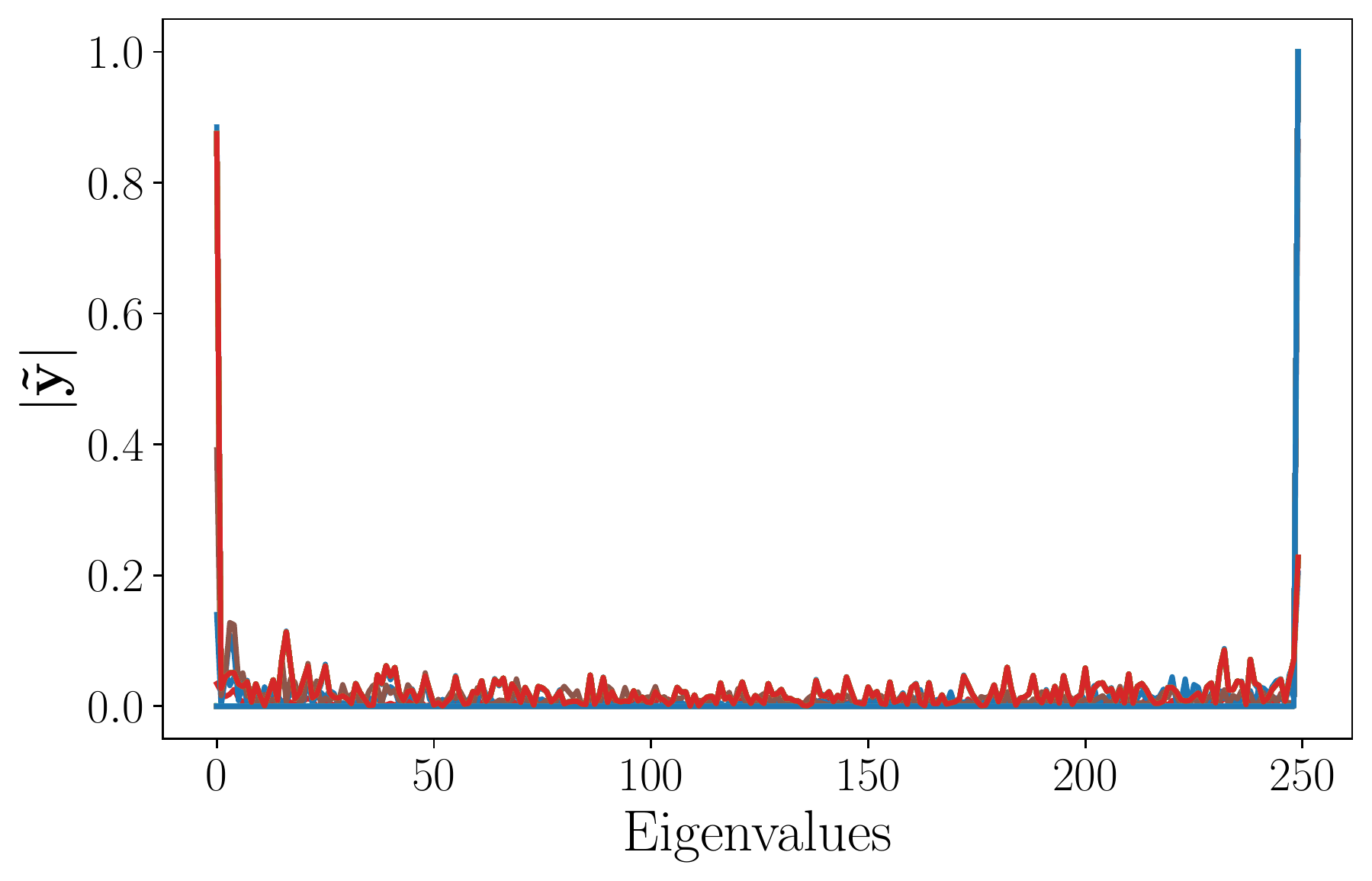}
    \caption{Output frequency of GCNN}
    \label{fig:app:starWars:outputFreqGCNN}
\end{figure}
\hfill
\begin{figure}
    \centering
    \includegraphics[width=0.9\columnwidth]{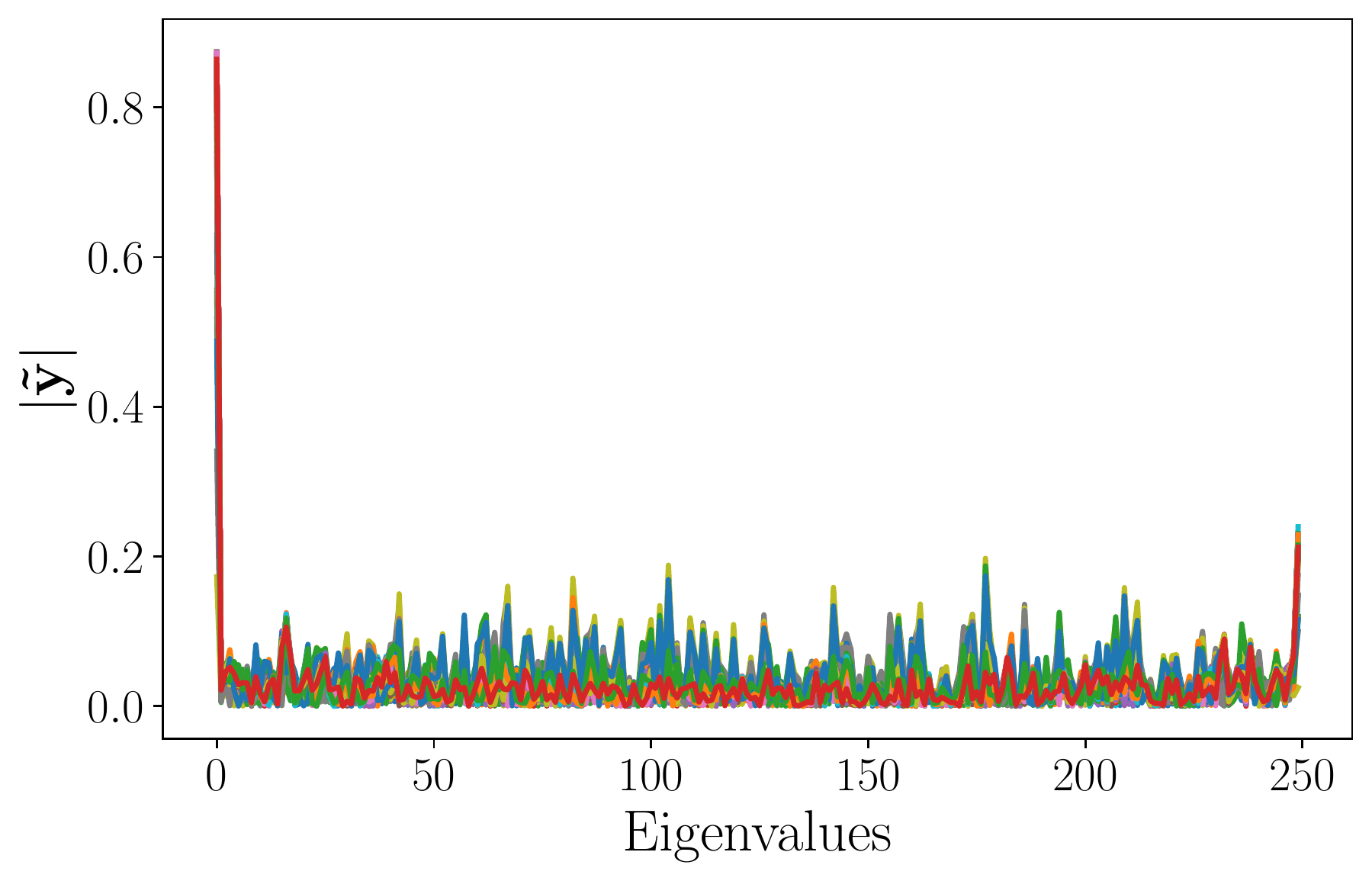}
    \caption{Output frequency of Learn NVGF}
    \label{fig:app:starWars:outputFreqLearn}
\end{figure}

